\documentclass[10pt]{article}
\usepackage[a4paper,left=2.5cm,right=2.5cm,top=2cm,bottom=2cm]{geometry}
\usepackage[latin1]{inputenc}
\usepackage[OT1]{fontenc}

\usepackage{natbib}
\bibpunct{(}{)}{;}{a}{,}{,}
\newenvironment{keywords}
{\bgroup\leftskip 25pt\rightskip 25pt \small\noindent{\bf Keywords:} }
{\par\egroup\vskip 0.25ex}
\usepackage{amsthm}
\theoremstyle{remark}

\usepackage{amsmath, amssymb,amsfonts, epsfig}
\newtheorem{mytheorem}{Theorem}
\newtheorem{mylemma}{Lemma}
\newtheorem{mydefinition}{Definition}
\usepackage{url}

\newcommand{\RR}{\mathbb{R}}
\newcommand{\V}{\ve{V}}
\newcommand{\W}{\ve{W}}
\renewcommand{\H}{\ve{H}}
\newcommand{\hh}{\ve{h}}
\newcommand{\vv}{\ve{v}}
\newcommand{\tih}{\tilde{\ve{h}}}

\newcommand{\ve}[1]{ {\mathbf{#1}} }

\newcommand{\defequal}{\stackrel{\mbox{\footnotesize def}}{=}}
\newcommand{\bal}{\bal}
\newcommand{\eal}{\end{eqnarray}}
\newcommand{\baln}{\begin{eqnarray*}}
\newcommand{\ealn}{\end{eqnarray*}}
\newcommand{\conv}[1]{\smash{\overset{\scriptscriptstyle\smile}{#1}}}
\newcommand{\conc}[1]{\smash{\overset{\scriptscriptstyle\frown}{#1}}}
\def\bal#1\eal{\begin{align}#1\end{align}}
\def\balx#1\ealx{\begin{align*}#1\end{align*}}

\def\argmin{\mathop{\mathrm{arg\,min}}}
\def\cst{\textit{cst}}

\usepackage{color}

\begin{document}
\title{Algorithms for nonnegative matrix factorization with the $\beta$-divergence}

\author{C\'edric F\'evotte\textsuperscript{1}  and J\'er\^ome Idier\textsuperscript{2}
\medskip\\
\textsuperscript{1} CNRS LTCI; T\'el\'ecom ParisTech, France \\
{\small Email: fevotte@telecom-paristech.fr} \\
\textsuperscript{2} CNRS IRCCyN; \'Ecole Centrale de Nantes, France \\
{\small Email: jerome.idier@irccyn.ec-nantes.fr}
}

\date{March 7, 2011\\ 
\bigskip
to appear in \emph{Neural Computation}}

\maketitle
\begin{abstract}
This paper describes algorithms for nonnegative matrix factorization (NMF) with the $\beta$-divergence ($\beta$-NMF). The $\beta$-divergence is a family of cost functions parametrized by a single shape parameter $\beta$ that takes the Euclidean distance, the Kullback-Leibler divergence and the Itakura-Saito divergence as special cases ($\beta = 2,1,0$ respectively). The proposed algorithms are based on a surrogate \emph{auxiliary function} (a local majorization of the criterion function). We first describe a \emph{majorization-minimization} (MM) algorithm that leads to multiplicative updates, which differ from standard heuristic multiplicative updates by a $\beta$-dependent power exponent. The monotonicity of the heuristic algorithm can however be proven for $\beta \in (0,1)$ using the proposed auxiliary function. Then we introduce the concept of \emph{majorization-equalization} (ME) algorithm which produces updates that move along constant level sets of the auxiliary function and lead to larger steps than MM. Simulations on synthetic and real data illustrate the faster convergence of the ME approach. The paper also describes how the proposed algorithms can be adapted to two common variants of NMF: penalized NMF (i.e., when a penalty function of the factors is added to the criterion function) and convex-NMF (when the dictionary is assumed to belong to a known subspace).
\end{abstract}

\begin{keywords} Nonnegative matrix factorization (NMF), $\beta$-divergence, multiplicative algorithms, majorization-minimization (MM), majorization-equalization (ME).
\end{keywords}

\section{Introduction} \label{sec:intro}

Given a data matrix $\ve{V}$ of dimensions ${F \times N}$ with nonnegative entries, NMF is the problem of finding a factorization
\begin{equation} \label{eqn:facto}
\V \approx \W \H
\end{equation}
where $\W$ and $\H$ are nonnegative matrices of dimensions $F \times K$ and $K \times N$, respectively. $K$ is usually chosen such that $F\,K + K\,N \ll F\,N$, hence reducing the data dimension. The factorization is in general only approximate, so that the terms ``approximate nonnegative matrix factorization'' or ``nonnegative matrix approximation" also appear in the literature. NMF has been used for various problems in diverse fields. To cite a few, let us mention the problems of learning parts of faces and semantic features of text \citep{lee99}, polyphonic music transcription \citep{sma03}, object characterization by reflectance spectra analysis \citep{ber07}, portfolio diversification \citep{dra07}, DNA gene expression analysis \citep{bru04,gao05}, clustering of protein interactions \citep{gre08}, image denoising and inpainting \citep{mair10}, etc. The factorization~\eqref{eqn:facto} is usually sought after through the minimization problem 
\begin{equation} \label{eqn:mini}
\underset{\W,\H}{\text{min}} \ D(\V | \W \H) \ \text{subject to} \ \W \ge 0, \H \ge 0
\end{equation}
where the notation $ \ve{A} \ge 0$ expresses nonnegativity of the entries of matrix $\ve{A}$ (and not semidefinite positiveness), and where $D(\V | \W \H)$ is a separable measure of fit such that
\begin{equation} \label{eqn:defcost}
D(\V | \W \H) = \sum_{f=1}^F \sum_{n=1}^N d( [\V]_{fn} | [\W \H]_{fn})
\end{equation}
where $d(x | y)$ is a scalar cost function. What we intend by ``cost function'' is a positive function of $y \in \mathbb{R}_+$ given $x \in \mathbb{R}_+$, with a single minimum for $x = y$.\\

A popular cost function in NMF is the $\beta$-divergence $d_\beta(x|y)$ of \citet{basu98,egu01,cic10}, defined rigorously in Section~\ref{sec:defbet}. In essence, it is a parameterized cost function with a single parameter $\beta$, which takes the Euclidean distance, the generalized Kullback-Leibler (KL) divergence and the Itakura-Saito (IS) divergence as special cases ($\beta = 2,1$ and $0$, respectively). NMF with the $\beta$-divergence has been widely used in music signal processing in particular, for transcription and source separation \citep{ogra07,ograd08,fitz09,icassp09b,neco09,vinc10,Dessein2010,henn10}. In these work the nonnegative data matrix $\V$ is a spectrogram which is decomposed into elementary spectra with NMF. The parameter $\beta$ can be tuned so as to optimize transcription or separation accuracy on training data. While popular in music signal processing, NMF with the $\beta$-divergence (shortened as ``$\beta$-NMF'' in the rest of the paper) can be of interest to any field:  the parameter $\beta$ essentially controls the assumed statistics of the observation noise and can either be fixed or learnt from training data or by cross-validation. As noted by \cite{eusipco09}, the values $\beta = 2,1,0$ respectively underly Gaussian additive, Poisson and multiplicative Gamma observation noise. The $\beta$-divergence offers  a continuum of noise statistics that interpolate between these three specific cases, see \citep{basu98,egu01,mina02, cic10}.\\

The standard $\beta$-NMF algorithm used in the above-mentioned papers is presented as a gradient-descent algorithm where the step size is set adaptively and chosen such that the updates are multiplicative, as originally described by \cite{cic06a}. The same algorithm can be derived from the following heuristic, proposed by \cite{neco09}. Let $\theta$ be a coefficient of $\W$ or $\H$. As will be seen later, when using the $\beta$-divergence the derivative $\nabla_\theta D(\theta)$ of the criterion $D(\V | \W \H)$ with respect to (w.r.t) $\theta$ can be expressed as the difference of two nonnegative functions such that $\nabla_\theta D(\theta) =  \nabla_\theta^+ D(\theta) - \nabla_\theta^- D(\theta)$. Then, a heuristic multiplicative algorithm simply writes
\bal \label{eqn:heur}
\theta \leftarrow \theta .\frac{  \nabla_\theta^- D(\theta)}{  \nabla_\theta^+ D(\theta)}
\eal
which ensures nonnegativity of the parameter updates, provided initialization with a nonnegative value. It produces a descent algorithm in the sense that $\theta$ is updated towards left (resp. right) when the gradient is positive (resp. negative). A fixed point $\theta^\star$ of the algorithm implies either $\nabla_\theta D(\theta^\star) = 0$ or $\theta^\star =0$. Monotonicity of this algorithm has been proven by \cite{kom07} for the specific range of values of $\beta$ for which the divergence $d_\beta(x|y)$ is convex w.r.t $y$ (i.e., $\beta \in [1,2]$, see Section~\ref{sec:defbet}). The proof is based on a \emph{majorization-minimization} (MM) procedure: an \emph{auxiliary function} is built and iteratively minimized for each column of $\H$ (given $\W$) and each row of $\W$ (given $\H$). The auxiliary function is built using Jensen's inequality, thanks to convexity of the cost for $\beta \in [1,2]$. However, it was observed in practice that the multiplicative algorithm~\eqref{eqn:heur} is still monotone (i.e., decreases the criterion function at each iteration) for values of $\beta$ out of the ``convexity'' interval $[1,2]$, though no proof is to avail. \\

This paper studies three descent algorithms for $\beta$-NMF, based on an auxiliary function that unifies existing auxiliary functions for the Euclidean distance and KL divergence \citep{depi93,lee01}, the ``generalized divergence" of \cite{kom07} and the IS divergence \citep{cao99}. This auxiliary function was also recently proposed independently by \cite{naka10}. The construction of the auxiliary function relies on the decomposition of the criterion function into its convex and concave parts, following the approach of \cite{cao99} for the IS divergence. An auxiliary function to the convex part is constructed using Jensen's inequality while the concave part is locally majorized by its tangent. It is shown that MM algorithms based on the latter auxiliary function yield multiplicative updates that coincide with the heuristic described by Eq.~\eqref{eqn:heur} for $\beta \in [1,2]$, but differ from a $\beta$-dependent power exponent when $\beta \not\in [1,2]$, a result also obtained by \cite{naka10}. Additionally, we show that the monotonicity of the heuristic algorithm can however be proven for $\beta \in (0,1)$, using the proposed auxiliary function (it is shown to produce a descent algorithm though it does not fully minimize the auxiliary function). Then we introduce the concept of \emph{maximization-equalization} (ME) algorithm which produces updates that move along constant level sets of the auxiliary function and leads to larger steps than MM. This is akin to \emph{overrelaxation} and is shown experimentally to produce faster convergence. Finally we show how the described MM, ME and heuristic algorithms can be adapted to two common variants of NMF: penalized NMF (i.e., when a penalty function of $\W$ or $\H$ is added to the criterion function) and ``convex''-NMF (when the dictionary is assumed to belong to a known subspace, as proposed by \cite{ding10}).\\

The paper is organized as follows. Section~\ref{sec:prelim} defines and discusses the $\beta$-divergence, and then exposes in details the optimization task addressed in this paper. Section~\ref{sec:betaauxfun} recalls the concept of auxiliary function and then introduces a general auxiliary function for the $\beta$-NMF problem. Section~\ref{sec:algos} describes algorithms based on the proposed auxiliary function, namely MM and ME algorithms, and describe how they relate to the heuristic update~\eqref{eqn:heur}. Section~\ref{sec:simus} reports simulations and convergence behaviors on synthetic and real data (with audio transcription and face interpolation examples). Section~\ref{sec:variants} describes extensions of the proposed algorithms to penalized and convex- NMF. Section~\ref{sec:conc} concludes and discusses open questions.

\section{Preliminaries} \label{sec:prelim}

In this section we present the $\beta$-divergence and more precisely specify the task that is addressed in this paper. A detailed exposition of the $\beta$-divergence can be found in \citep{cic10}.

\subsection{Definition of the $\beta$-divergence} \label{sec:defbet}

The $\beta$-divergence was introduced by \citet{basu98} and \citet{egu01} and can be defined as
\begin{equation} \label{eqn:beta}
d_{\beta}(x | y) \defequal
\left\lbrace
\begin{array}{cc}
\frac{1}{\beta\,(\beta-1)} \left( x^{\beta} + (\beta-1)\,y^{\beta} - \beta \, x\, y^{\beta-1} \right) &  \beta \in \mathbb{R} \backslash \{0,1\}  \\
x\, \log \frac{x}{y} - x + y & \beta = 1 \\
\frac{x}{y} - \log \frac{x}{y} - 1 & \beta = 0
\end{array}
\right.
\end{equation}
\citet{basu98} and \citet{egu01} assume $\beta > 1$, but the definition domain can be extended to $\beta  \in \mathbb{R}$, as suggested by \cite{cic06a}, which is the definition domain that is considered in this paper. The $\beta$-divergence can be shown continuous in $\beta$ by using the identity $\lim_{\beta \rightarrow 0} {(x^\beta-y^\beta)}/{\beta} = \log (x/y)$. The limit cases $\beta = 0$ and $\beta = 1$ correspond to the IS and KL divergences, respectively. The $\beta$-divergence coincides up to a factor $1 / \beta$ with the ``generalized divergence'' of \cite{kom07} which, in the context of NMF as well, was separately constructed so as to interpolate between the KL divergence ($\beta = 1$) and the Euclidean distance ($\beta = 2$). The $\beta$-divergence is plotted for various values of $\beta$ on Figure~\ref{fig:betadiv}. Note that in this paper we will abusively refer to $d_{\beta =2} = (x-y)^2 /2$ as the Euclidean distance, though the latter is formally defined with a square root, and for vectors.

\begin{figure}[t]
\begin{center}
\begin{tabular}{cc}
(a) $\beta<0$ & (b) $0\leq\beta<1$ \\
\epsfig{figure=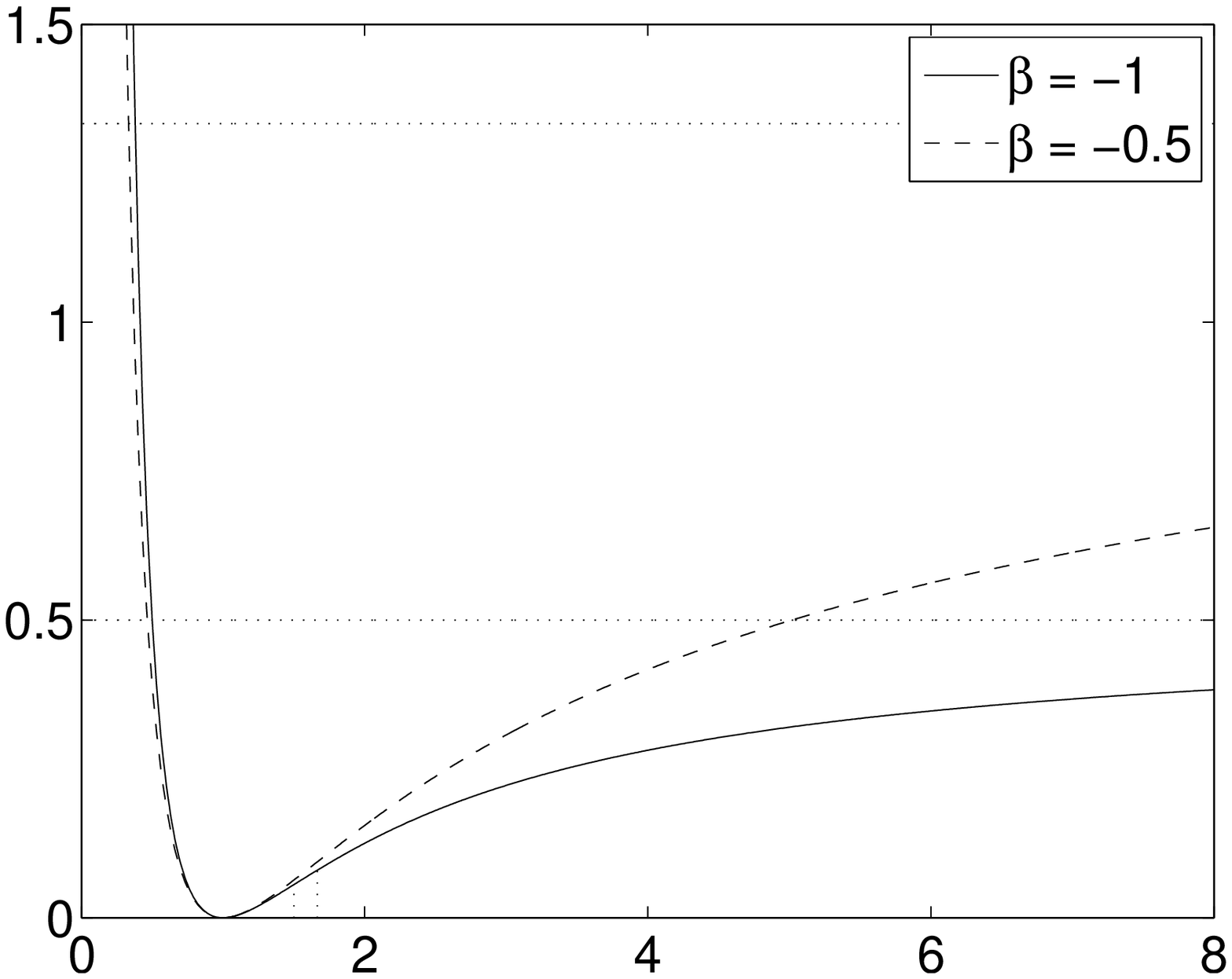,width=0.34\linewidth} & \epsfig{figure=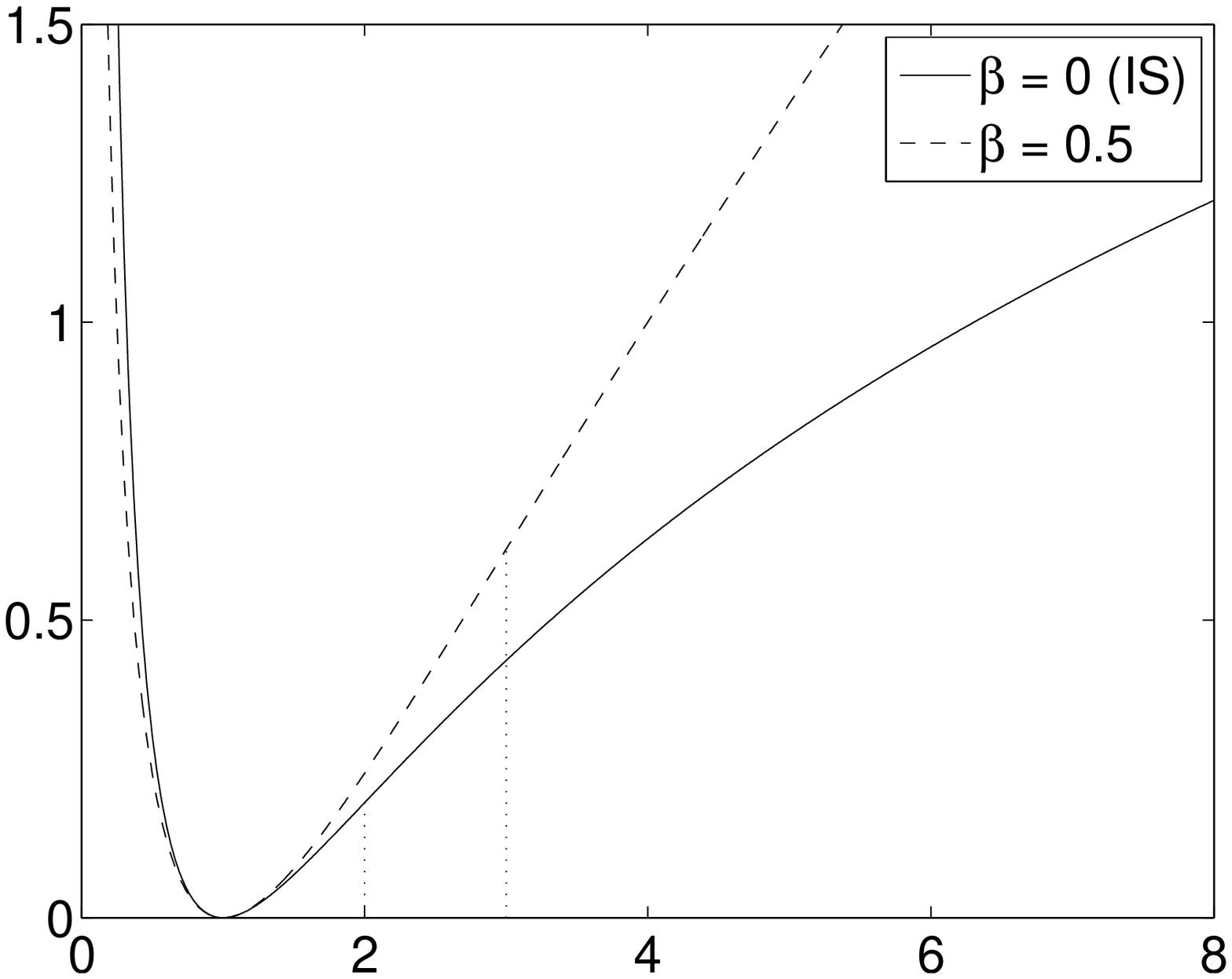,width=0.34\linewidth} \\
(c) $\beta=1$ & (d) $1 < \beta\leq2$ \\
\epsfig{figure=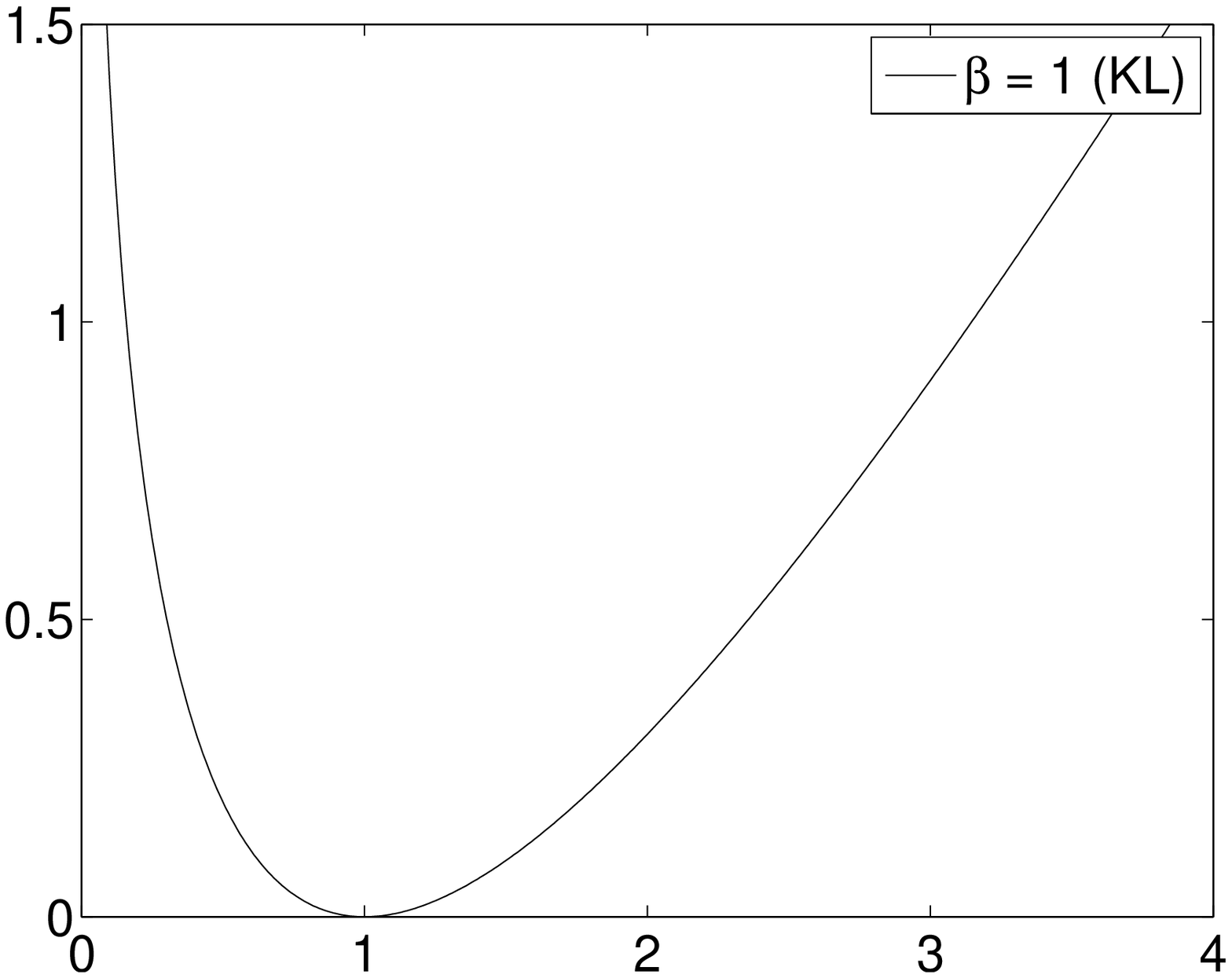,width=0.34\linewidth} & \epsfig{figure=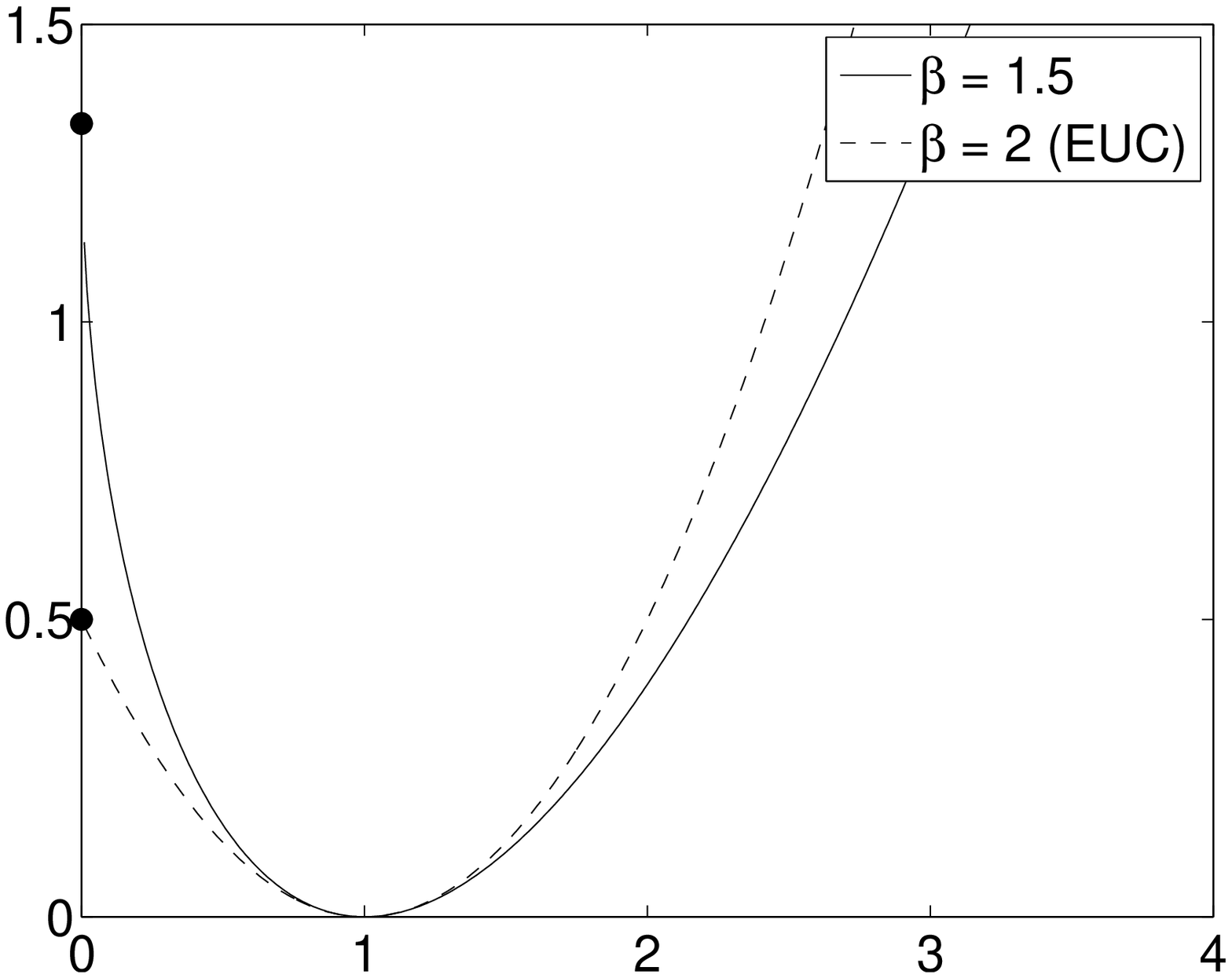,width=0.34\linewidth} \\
(e) $\beta>2$ & \\
\epsfig{figure=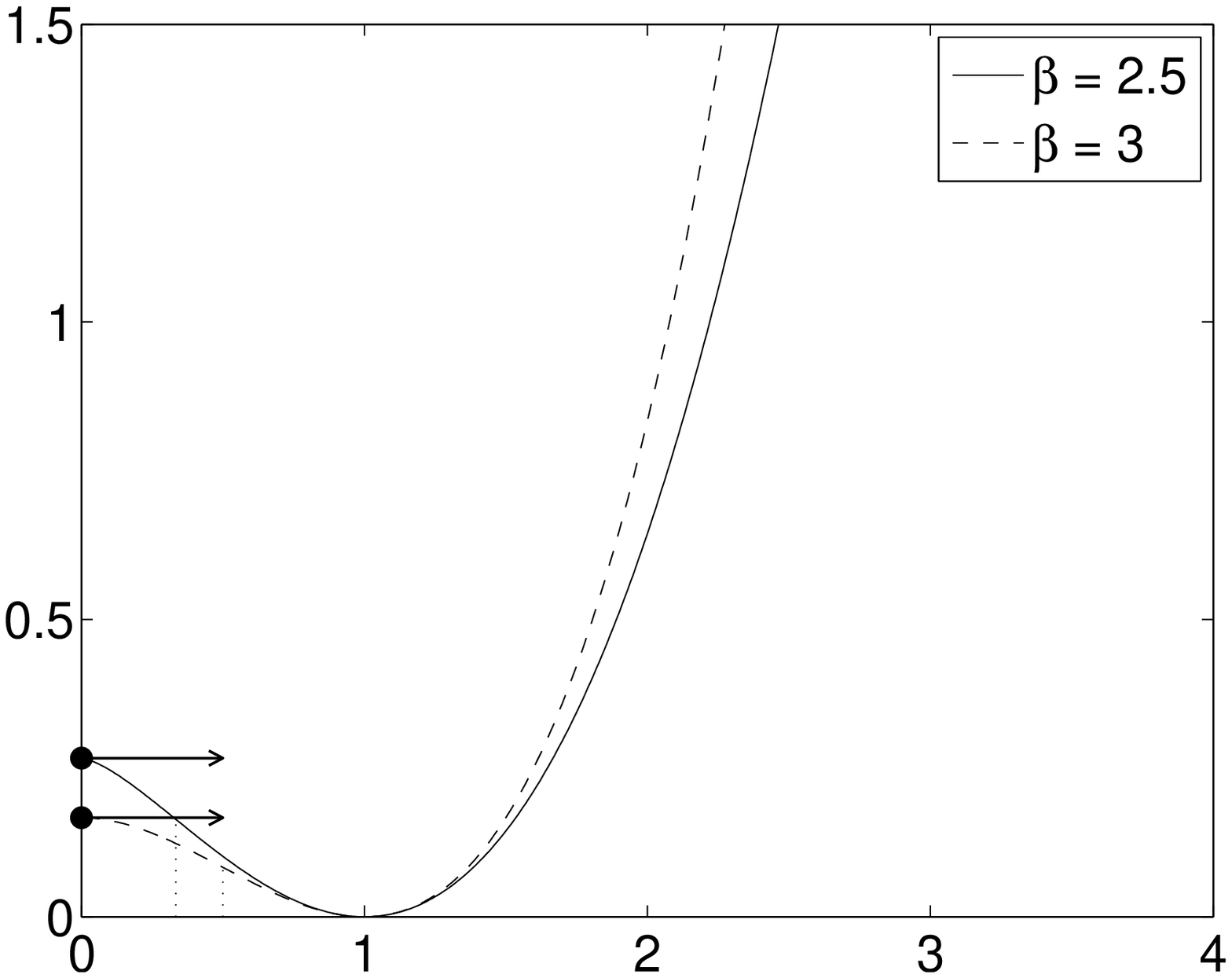,width=0.34\linewidth} & 
\end{tabular}
\end{center}
\caption{$\beta$-divergence $d_{\beta}(x | y)$ as a function of $y$ (with $x=1$). The subfigures illustrate the regimes of the $\beta$-divergence for its five characteristic ranges of values of $\beta$. The divergence is convex for $1 \leq \beta \leq 2$, as seen on subfigures (c) and (d). On the other subfigures, the inflection points are indicated with vertical dotted lines. For $\beta < 0$, the divergence possesses horizontal asymptotes of coordinate $x^\beta / (\beta (\beta -1))$ as $y \rightarrow \infty$. For $\beta > 1$, the divergence takes finite value $x^\beta / (\beta (\beta -1))$ at $y=0$, where the derivative is zero for $\beta >2$.
}
\label{fig:betadiv}
\end{figure}

The first and second derivative of $d_{\beta}(x|y)$ w.r.t $y$ are continuous in $\beta$, and write
\bal
d'_{\beta}(x|y) &= y^{\beta-2} \, (y-x), \label{eqn:deriv} \\
d''_{\beta}(x|y) &= y^{\beta-3} \, \left[ (\beta-1) y - (\beta-2) x \right].
\eal
The derivative shows that $d_{\beta}(x|y)$, as a function of $y$, has a single minimum in $y = x$ and that it increases with $|y-x|$, justifying its relevance as a measure of fit. The second derivative shows that the $\beta$-divergence is convex w.r.t $y$ for $\beta \in [1,2]$.
Outside this interval the divergence can always be expressed as the sum of a convex, concave and constant part, such that
\bal
d_\beta(x|y) = \conv{d}(x|y) + \conc{d}(x|y) + \bar{d}(x)
\label{eq:3c}
\eal
where $\conv{d}(x|y)$ is a convex function of $y$, $\conc{d}(x|y)$ is a concave function of $y$ and $\bar{d}(x)$ is a constant of $y$. {The decomposition is not unique, since constant or linear terms (w.r.t $y$) are both convex and concave, or, less trivially, since any convex term can be added to $\conv{d}(x|y)$ while subtracted from $\conc{d}(x|y)$.}
In the following we will use the ``natural conventions'' given in Table~\ref{tab:beta}. 

\begin{table}[t]
\centering
\renewcommand{\arraystretch}{1.4}
\begin{tabular}{|c|c|c||c|c||c|}
\hline
 & $\conv{d}(x|y)$ & $\conv{d}'(x|y)$ & $\conc{d}(x|y)$ & $\conc{d}'(x|y)$ & $\bar{d}(x)$ \\
\hline
$\beta < 1 $ and $\beta \not= 0$ &  $- \frac{1}{\beta-1} x \,y^{\beta -1} $ &  $- x\, y^{\beta-2} $& $\frac{1}{\beta} y^\beta  $  & $ y^{\beta-1}$ & $\frac{1}{\beta (\beta-1)} x^\beta $\\
\hline
$\beta = 0$ & $ x \, y^{-1} $ & $- x \, y^{-2}$ & $\log{y} $ & $y^{-1}$ & $ x (\log x - 1) $\\
\hline
$1 \le \beta \le 2$ & $d_\beta(x|y)$ & $d_\beta'(x|y)$ & 0 & 0 & 0  \\
\hline
$\beta >2$  &  $\frac{1}{\beta} y^\beta  $ & $ y^{\beta-1} $ & $- \frac{1}{\beta-1} x \, y^{\beta -1} $ & $- x\, y^{\beta-2}$ & $\frac{1}{\beta (\beta-1)} x^\beta$ \\
\hline
\end{tabular}
\caption{Example of differentiable convex-concave-constant decomposition of the $\beta$-divergence under the form \eqref{eq:3c}.}
\label{tab:beta}
\end{table}

As noted by \cite{neco09}, a noteworthy property of the $\beta$-divergence is its behavior w.r.t to scale, as the following equation holds for any value of $\beta$:
\begin{equation}
d_{\beta}(\lambda \, x | \lambda \, y) = \lambda^{\beta} \ d_{\beta}(x | y).
\end{equation}
It implies that factorizations obtained with $\beta > 0$ (such as with the Euclidean distance or the KL divergence) will rely more heavily on the largest data values and less precision is to be expected in the estimation of the low-power components, and conversely factorizations obtained with $\beta < 0$ will rely more heavily on smallest data values. The IS divergence ($\beta = 0$) is scale-invariant, i.e., $d_{IS}(\lambda \, x | \lambda \, y) = d_{IS}( x |  y)$, and is the only one in the family of $\beta$-divergences to possess this property. Factorizations with small positive values of $\beta$ are relevant to decomposition of audio spectra, which typically exhibit exponential power decrease along frequency $f$ and also usually comprise low-power transient components such as note attacks together with higher power components such as tonal parts of sustained notes. For example, \cite{neco09} present the results of the decomposition of a piano power spectrogram with IS-NMF and show that components corresponding to very low residual noise and hammer hits on the strings are extracted with great accuracy, while these components are either ignored or severely degraded when using Euclidean or KL divergences. Similarly, the value $\beta = 0.5$ is advocated by \cite{fitz09,henn10} and has been shown to give optimal results in music transcription based on NMF of the magnitude spectrogram by \cite{vinc10}.

The $\beta$-divergence belongs to the family of Bregman divergences. For $\beta \not\in \{0,1\}$, a suitable Bregman generating function is $\phi(y) = y^\beta / (\beta (\beta-1))$, as noted by \cite{eusipco09}. This function, however, cannot generate the IS and KL divergences by continuity when $\beta$ tends to 0 or 1. The latter divergences may nonetheless be generated ``separately'', using the functions $\phi(y) = -\log y$ and $\phi(y) =  y \log y$, respectively. \cite{cic10} give a general Bregman generating function of the $\beta$-divergence, defined for all $\beta \in \mathbb{R}$, in the form of $\phi_{\beta\not=0,1}(y) = (y^\beta - \beta y + \beta -1) / (\beta (\beta-1))$, $\phi_{\beta=0}(y) = y - \log y -1$ and $\phi_{\beta=1}(y) = y \log y - y +1$. NMF with Bregman divergences has been considered by \cite{dhi05}, where the lack of results about the monotonicity of multiplicative algorithms in general has been noted.\footnote{More precisely, \cite{dhi05} give proofs of monotonicity for the ``reverse'' problem of minimizing $D(\W \H | \V)$ instead of $D(\V | \W \H)$, while pointing that monotonicity of multiplicative algorithms based on the heuristic~\eqref{eqn:heur} for the latter problem is however observed in practice.} This paper fills this gap for the specific case of $\beta$-divergence.

\subsection{Task}

\paragraph{Core optimization problem} As to our best knowledge all algorithms in the literature to date, the NMF algorithms we describe in this paper sequentially update $\H$ given $\W$ and then $\W$ given $\H$. These two steps are essentially the same, by symmetry of the factorization ($\V\approx \W \H$ is equivalent to $\V^T \approx \H^T \W^T$ and the roles of $\W$ and $\H$ are simply exchanged), and because we are not making any assumption on the relative values of $F$ and $N$. Hence, we may concentrate on solving the following subproblem
\begin{equation} \label{eqn:mini2}
\underset{\H}{\text{min}} \ C(\H) \defequal D(\V | \W \H) \ \text{subject to} \ \H \ge 0
\end{equation}
with fixed $\W$ and where in the rest of the paper $D(\V | \W \H)$ is as of Eq.~\eqref{eqn:defcost} with $d(x|y) = d_\beta(x|y)$. The criterion function $C(\H)$ separates into $\sum_n D(\ve{v}_n | \W \ve{h}_n)$, where $\ve{v}_n$ and $\ve{h}_n$ are the $n^{th}$ row of $\V$ and $\H$, respectively, so that we are essentially left with solving the problem
\begin{equation} \label{eqn:minih}
\underset{\hh}{\text{min}} \ C(\hh) = D(\vv | \W \hh) \ \text{subject to} \ \hh \ge 0
\end{equation}
where $\vv \in \mathbb{R}_+^{F}$, $\W \in \mathbb{R}_+^{F \times K}$ and $\hh \in \mathbb{R}_+^{K}$.

\paragraph{KKT necessary conditions} An admissible solution $\ve{h}^{\star}$ to problem~\eqref{eqn:minih} must satisfy the Karush-Kuhn-Tucker (KKT) first order optimality conditions, which write
\bal
\nabla_{\ve{h}} C(\ve{h}^{\star}) . \ve{h}^{\star} &=  0 \label{eqn:kkt1} \\
\nabla_{\ve{h}} C(\ve{h}^{\star}) &\ge 0  \label{eqn:kkt2}\\
\ve{h}^{\star} &\ge 0 \label{eqn:kkt3}
\eal
where the dot notation `$.$' denotes entrywise operations (here term-to-term multiplication) and $\nabla_{\ve{h}} C(\ve{h})$ denotes the gradient of $C(\hh)$, given by
\bal
\nabla_{\ve{h}} C(\ve{h}) &= \ve{W}^T \, [d'(v_f | [\W \hh]_f)]_f \\
&= \W^T [ (\W \hh)^{.(\beta-2)} (\W \hh - \ve{v})]
\label{eq:gradient}
\eal
where the notation $[x_f]_f$ refers to the column vector $[x_1,\ldots,x_F]^T$. The KKT conditions~\eqref{eqn:kkt1}-\eqref{eqn:kkt3} can be summarized as
\begin{equation} \label{eqn:kkt}
\text{min} \{ \ve{h}^\star, \nabla_{\ve{h}} C(\ve{h}^{\star}) \} = \ve{0}_K
\end{equation}
where the ``min'' operator is entrywise and $\ve{0}_K$ is a null vector of dimension $K$. 

\paragraph{Algorithms} In the following, we will say that an algorithm is \emph{monotone} if it produces a sequence of iterates $\{ \hh^{(i)} \}_{i \ge 0}$, such that $C(\hh^{(i+1)}) \le C(\hh^{(i)})$ for all $i \ge 0$. An algorithm is said \emph{convergent} if it produces a sequence of iterates $\{ \hh^{(i)} \}_{i \ge 0}$ which converges to a limit point $\hh^{\star}$ satisfying the KKT conditions~\eqref{eqn:kkt1}-\eqref{eqn:kkt3}. Monotonicity does not imply convergence in general, nor is monotonicity necessary to convergence.

\section{An auxiliary function for $\beta$-NMF} \label{sec:betaauxfun}

In this section we properly define the concept of auxiliary function and then exhibit a separable auxiliary function for the $\beta$-NMF problem.

\subsection{Definition of auxiliary function}
\label{subsec:defaux}

\begin{mydefinition}[Auxiliary function] The $\RR_+^{K} \times \RR_+^{K} \rightarrow \RR_+$ mapping $G(\ve{h} | \tilde{\ve{h}} )$ is said to be an \emph{auxiliary function} to $C(\hh)$ if and only if 
\begin{itemize}
\item $\forall \hh \in \RR_+^K$,  $C(\ve{h}) = G(\ve{h} | \ve{h})$
\item $\forall (\hh,\tih) \in \RR_+^K \times \RR_+^K$, $C(\ve{h}) \le G(\ve{h} | \tilde{\ve{h}}) $ 
\end{itemize}
\end{mydefinition}

\bigskip

In other words an auxiliary function $G(\ve{h} | \tih)$ is a \emph{majorizing function} or \emph{upper bound} of $C(\hh)$ which is tight for $\hh = \tih$. The optimization of $C(\hh)$ can be replaced by iterative optimization of $G(\hh | \tih)$. Indeed, any iterate $\hh^{(i+1)}$ satisfying  
\bal \label{eqn:idea}
 G(\hh^{(i+1)} | \hh^{(i)} ) \le G(\hh^{(i)} | \hh^{(i)})
 \eal
satisfies $C(\hh^{(i+1)}) \le C(\hh^{(i)})$, because we have
\bal \label{eqn:algo}
C(\hh^{(i+1)}) \le G(\hh^{(i+1)} | \hh^{(i)}) \le G(\hh^{(i)} | \hh^{(i)}) = C(\hh^{(i)}).
\eal
The iterate $\hh^{(i+1)} $ is typically chosen as
\begin{equation} \label{eqn:max}
\ve{h}^{(i+1)} = \argmin_{\ve{h} \ge 0} \ G(\ve{h} | \ve{h}^{(i)})
\end{equation}
which forms the basis of \emph{maximization-minimization} (MM) algorithms, see, e.g., \citep{hunt04} for a tutorial. However, any other iterate $\hh^{(i+1)}$ satisfying~\eqref{eqn:idea} produces a monotone algorithm. {As such, Figure~\ref{fig:ME} illustrates the three updates strategies that will be developed in this paper.}

\begin{figure}[t]
\centering
\epsfig{figure=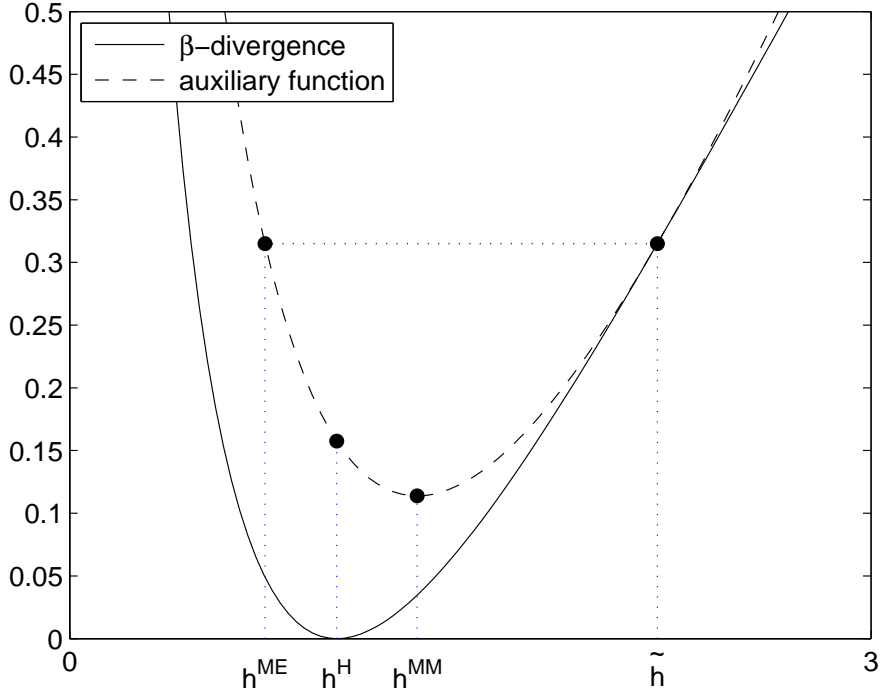,width=0.85\linewidth}
\caption{
{
The $\beta$-divergence $d_\beta(x|y)$ for $\beta = 0.5$ (with $x =1$) and its auxiliary function in dimension one (with $\tilde{h}$ = 2.2). The MM update $h^{\text{MM}}$ corresponds to the minimum of the auxiliary function, see Section~\ref{sec:mm}. The heuristic update $h^{\text{H}}$ given by Eq.~\eqref{eqn:heur} is discussed in Section~\ref{sec:heur} (the heuristic update minimizes the criterion function in the simple one-dimensional case but this is not true in larger dimensions). The ME update $h^{\text{ME}}$ consists in selecting the next update ``beyond the valley'' defined by the auxiliary function, from the current solution $\tilde{h}$, see Section~\ref{sec:me}.}}
\label{fig:ME}
\end{figure}

\subsection{Separable auxiliary function for $\beta$-NMF} \label{sec:auxfun}

In this section we construct an auxiliary function to $C(\hh)$ for the specific case of the $\beta$-divergence. Our approach follows the one of \cite{cao99} for IS divergence, and consists of majorizing the convex part of the criterion using Jensen's inequality and majorizing the concave part by its tangent, as detailed in the proof of the following theorem. Here and henceforth, we denote $\W \tih$ by $\tilde{\ve{v}}$, with entries $[\W \tih]_f=\tilde{v}_f$.

\begin{mytheorem}[Auxiliary function for $\beta$-NMF] \label{theo:auxfunc}
Let $\tih$ be such that 
\begin{itemize}
\item[(i)]  $\forall f, \tilde{v}_f > 0$
\item[(ii)] $\forall k, \tilde{h}_{k} > 0$
\end{itemize}
Then, the function $G(\hh|\tih)$ defined by 
\bal
G(\hh | \tih) = \sum_f  \left[\sum_{k} \frac{w_{fk} \tilde{h}_{k}}{\tilde{v}_f} \ \conv{d} \left( v_f |  \tilde{v}_f \frac{h_k}{\tilde{h}_k} \right) \right]  + \left[ \conc{d}'(v_f|\tilde{v}_f) \sum_k w_{fk}  (h_k - \tilde{h}_k) + \conc{d}(v_f | \tilde{v}_f) \right] + \bar{d}(v_f) \label{eqn:auxbeta}
\eal
is an auxiliary function to $C(\hh) = \sum_f d(v_f| [\W \hh]_f)$, where $ \conv{d}(x|y) + \conc{d}(x|y) + \bar{d}(x)$ is any differentiable convex-concave-constant decomposition of the $\beta$-divergence, such as the one defined in Table~\ref{tab:beta}.
\end{mytheorem}

\bigskip

\begin{proof} The condition $G(\ve{h}|{\ve{h}}) = C(\hh) $ is trivially met. The criterion $C(\hh)$ may be written as
\bal
C(\ve{h}) &= \sum_f C_f(\ve{h})
\eal
where $C_f(\ve{h}) \defequal d(v_f | [\ve{W} \ve{h}]_f)$. We prove $C(\ve{h}) \le G(\ve{h} | \tilde{\ve{h}})$ by constructing an auxiliary function to each part $C_f(\hh)$ of the criterion, and more precisely by treating the convex and concave part separately. Let us define  $\conv{C}_f(\hh) \defequal \conv{d}(v_f | [\ve{W} \hh]_f)$ and $\conc{C}_f(\hh) \defequal \conc{d}(v_f | [\ve{W} \hh]_f)$, so that we can write
\begin{equation}
C_f(\hh) =  \conv{C}_f(\hh) + \conc{C}_f(\hh) + \bar{d}(v_f).
\label{eq:c3}
\end{equation}

\noindent \emph{Convex part}:  We first prove that
\bal \label{eqn:auxconvf}
\conv{G}_f(\ve{h}|\tilde{\ve{h}}) = \sum_{k} \frac{w_{fk} \tilde{h}_{k}}{\tilde{v}_f} \ \conv{d} \left( v_f |  \tilde{v}_f \frac{h_k}{\tilde{h}_k} \right)
\eal
is an auxiliary function to $\conv{C}_f(\hh)$. The condition $\conv{G}_f(\ve{h}|{\ve{h}}) = \conv{C}_f(\hh) $ is trivially met. The condition $\conv{G}_f(\hh|\tih) \ge \conv{C}_f(\tih)$ is proven as follows. Let $\mathcal{K}$ be the set of indices $k$ such that $w_{fk} \not= 0$. Define $\forall k \in \mathcal{K}$,
\bal
\tilde{\lambda}_{fk} = \frac{w_{fk} \tilde{h}_{k}}{\tilde{v}_f} = \frac{w_{fk} \tilde{h}_{k}}{\sum_{\ell \in \mathcal{K}} w_{f\ell} \tilde{h}_{\ell}}.
\eal
We have $\sum_{k \in \mathcal{K}} \tilde{\lambda}_{fk} = 1$ and 
\bal
\conv{G}_f(\ve{h}|\tilde{\ve{h}}) &=  \sum_{k \in \mathcal{K}} \tilde{\lambda}_{fk} \ \conv{d} \left( v_f | \frac{w_{fk} h_{k}}{ \tilde{\lambda}_{fk}} \right) \\
& \ge  \conv{d} \left( v_f | \sum_{k \in \mathcal{K}} \tilde{\lambda}_{fk} \frac{w_{fk} h_{k}}{ \tilde{\lambda}_{fk}} \right) \\
&= \conv{d} \left( v_f | \sum_{k=1}^K w_{fk} h_{k} \right) \\
&= \conv{C}_f(\hh)
\eal
where we used Jensen's inequality, by convexity of $\conv{d}(x|y)$.\\

\noindent \emph{Concave part}: 
An auxiliary function $\conc{G}_{f}(\hh|\tih)$ to the concave part $\conc{C}_f(\hh)$ can be taken as the first order Taylor approximation to $\conc{C}_f(\hh)$ in the vicinity of $\tih$, i.e.,
\begin{equation} \label{eqn:auxbetaconc}
\conc{G}_{f}(\hh|\tih) = \conc{C}_f(\tih) + \nabla^T \conc{C}_f(\tih)  \ (\hh - \tih). 
\end{equation}
The function satisfies $\conc{G}_{f}(\hh|\hh) = \conc{C}_f(\hh)$ by construction and $\conc{G}_{f}(\hh|\tih) \ge \conc{C}_f(\hh)$ by concavity of $\conc{C}_f(\hh)$, using the property that the tangent to any point is an upper bound of a concave function.\footnote{$\conc{C}_f(\hh)= \conc{d}(v_f | [\ve{W} \hh]_f)$ is concave as the composition of a concave function and a linear function.}
Using
\bal
\nabla_{h_k} \conc{C}_f(\hh) = w_{fk} \conc{d}'(v_f|[\W \hh]_f)
\eal
the explicit form for $\conc{G}_{f}(\hh|\tih)$ is given by
\bal
\conc{G}_{f}(\hh|\tih) = \conc{d}(v_f | \tilde{v}_f) + \conc{d}'(v_f|\tilde{v}_f) \sum_k w_{fk}  (h_k - \tilde{h}_k).
\eal

In the end a suitable auxiliary function $G(\hh | \tih)$ to $C(\hh)$ is obtained by summing up the auxiliary functions constructed for each individual part of the criterion, i.e.,
\begin{equation}
G(\hh | \tih) = \sum_f \left( \conv{G}_{f}(\hh|\tih) + \conc{G}_{f}(\hh|\tih) + \bar{d}(v_f) \right)
\end{equation}
which leads to Eq.~\eqref{eqn:auxbeta}.
\end{proof}

\paragraph{Properties of the auxiliary function} 
$G(\hh | \tih)$ is by construction separable in functions of the individual coefficients $h_k$ of $\hh$, which allows to decouple the optimization. It is convenient to rewrite the auxiliary function as such in order to derive some of the algorithms of Section~\eqref{sec:algos}. We may write
\bal
G(\hh | \tih) = \sum_k G_k(h_k | \tih) + \cst
\eal
where $\cst$ is a constant w.r.t $\hh$ and 
\bal
G_k(h_k | \tih) \defequal \tilde{h}_k \left[\sum_f \frac{w_{fk} }{ \tilde{v}_f} \conv{d}\left(v_f |\tilde{v}_f \frac{h_k}{\tilde{h}_k}\right)\right]  + h_k \left[ \sum_f w_{fk}\conc{d}'(v_f| \tilde{v}_f)\right] .
\label{eq:Gk}
\eal
The gradient of the auxiliary function is given by
\bal \label{eqn:gradconvconc}
\nabla_{h_k} G(\hh | \tih) = \sum_f w_{fk} \, \left[ \conv{d}' \left(v_f |  \tilde{v}_f \frac{h_k}{\tilde{h}_k} \right) + \conc{d}'(v_f|\tilde{v}_f) \right].
\eal
Thanks to the separability of the auxiliary function into its variables the Hessian matrix is diagonal with
\bal 
\nabla_{h_k}^2 G(\hh | \tih) = \sum_f   \tilde{v}_f \frac{w_{fk}}{\tilde{h}_k}  \, \conv{d}'' \left(v_f |  \tilde{v}_f \frac{h_k}{\tilde{h}_k} \right).
\eal
By convexity of $\conv{d}(x|y)$ we have  $\conv{d}''(x|y) \ge 0$ which implies positive definiteness of the Hessian matrix and hence convexity of the auxiliary function $G(\hh | \tih)$ (convexity more simply derives from the fact that the auxiliary function is built as a sum of convex functions). \\

\paragraph{Connections with other works} The construction of $G(\hh|\tih)$ employs standard mathematical tools (Jensen's inequality, Taylor approximation) that are well known from the MM literature, see, e.g., \citep{hunt04}. For $\beta \in [1,2]$, $G(\hh|\tih)$ coincides with the auxiliary function built by \cite{kom07}. This latter paper proposed itself a generalization of the auxiliary functions proposed by \cite{lee01} for the Euclidean distance ($\beta=2$) and the generalized KL divergence ($\beta=1$). For $\beta = 0$ (IS divergence), $G(\hh|\tih)$ coincides with the auxiliary function proposed by \cite{cao99}. It is worth recalling that in the algorithms proposed by \cite{lee01} the update of $\W$ given $\H$ or $\H$ given $\W$ are instances of well known algorithms for image restoration (for which $\W$ acts as a fixed, known blurring matrix and $\H$ is a vectorized image to be reconstructed). These algorithms are the Iterative Space Reconstructing Algorithm (ISRA) of \cite{daub86} and the Richardson-Lucy (RL) algorithm of \cite{rich72,lucy74}, which perform nonnegative linear regression with the Euclidean distance and KL divergence, respectively. The ISRA and RL algorithms are shown to be MM algorithms by \cite{depi93}. Similarly, the algorithms proposed by \cite{cao99} for nonnegative linear regression with the IS divergence were designed in the image restoration setting. Finally, let us mention that an auxiliary function based on Jensen's inequality for NMF with the $\alpha$-divergence (which is always convex w.r.t to its second argument) is given by \cite{cic08b}.

\section{Algorithms for $\beta$-NMF} \label{sec:algos}

In this section we describe algorithms for $\beta$-NMF based on the auxiliary function constructed in the latter section. In the following $\tih$ should be understood as the current iterate $\hh^{(i)}$ and we are seeking to obtain $\hh^{(i+1)}$ such that Eq.~\eqref{eqn:idea} is satisfied.

\subsection{Maximization-Minimization (MM) algorithm} \label{sec:mm}

An MM algorithm can be derived by minimizing the auxiliary function $G(\hh|\tih)$ w.r.t to $\hh$. Given the convexity and the separability of the auxiliary function the optimum is obtained by canceling the gradient given by Eq.~\eqref{eqn:gradconvconc}. This is trivially done and leads to the following update:
\bal \label{eqn:upbeta}
h_k^\text{MM} = \tilde{h}_k \left(  \frac{\sum_f w_{fk} \, v_f \, \tilde{v}_f^{\beta-2} }{\sum_f w_{fk} \, \tilde{v}_f^{\beta-1}} \right)^{\gamma(\beta)}
\eal
where ${\gamma(\beta)}$ is given in Table~\ref{tab:gamma}.
\begin{table}[t]
\begin{center}
\begin{tabular}{|c|c|c|c|}
\hline
                            &  $\beta <1$ & $1 \le \beta \le 2$ & $\beta >2$ \\
\hline
$\gamma(\beta)$ &  ${\frac{1}{2-\beta}}$ & 1 &  $\frac{1}{\beta-1}$ \\
\hline
\end{tabular}
\end{center}
\caption{Exponent in the multiplicative updates given by the MM algorithm.}
\label{tab:gamma}
\end{table}
Note that $\gamma(\beta) \le 1, \forall \beta$. As suggested in Section~\ref{sec:intro}, the gradient of the criterion may be written as the difference of two nonnegative functions such that
\bal
\nabla_{h_k} C(\tih) &= \nabla_{h_k}^+ C(\tih) -  \nabla_{h_k}^- C(\tih) \\
\nabla_{h_k}^+ C(\tih) &= \sum_f w_{fk} \, \tilde{v}_f^{\beta-1}  \\
\nabla_{h_k}^- C(\tih) &=  \sum_f w_{fk} \, v_f \, \tilde{v}_f^{\beta-2}
\eal
so that the update \eqref{eqn:upbeta} can be rewritten in the more interpretable form
\bal
h_k^\text{MM} = \tilde{h}_k \left(  \frac{ \nabla_{h_k}^- C(\tih) }{ \nabla_{h_k}^+ C(\tih)} \right)^{\gamma(\beta)}.
\eal
The conclusion is thus that the MM algorithm leads to multiplicative updates, but the latter differ from the ``usual ones'', obtained by setting $\gamma(\beta) =1$ for all $\beta$ and derived heuristically by \cite{cic06a} through gradient descent with adaptative step or by \cite{neco09} by splitting the gradient into two nonnegative functions as discussed above and in Section~\ref{sec:intro}. The MM update differs from the heuristic update by the exponent $\gamma(\beta)$ which is not equal to one for $\beta \not\in [1,2]$.

\subsection{Heuristic algorithm} \label{sec:heur}

This section discusses the properties of the heuristic update proposed by \cite{cic06a,neco09} and defined for all $\beta$ by
\bal
h_k^{\text{H}} = \tilde{h}_k \left(  \frac{\sum_f w_{fk} \, v_f \, \tilde{v}_f^{\beta-2} }{\sum_f w_{fk} \, \tilde{v}_f^{\beta-1}} \right).
\label{eq:heur_update}
\eal
Very few mathematical results exist for the heuristic update when $\beta$ falls outside $[1,2]$, i.e., when the $\beta$-divergence $d_{\beta}(x | y)$ is not convex. In such a case, the heuristic update can be erroneously interpreted as an MM algorithm by wrongly applying Jensen's inequality to $C(\hh)$.
Yet, in the particular case $\beta=0$, it holds that each heuristic update produces a decrease of $C(\hh)$ \citep{cao99}. One objective of this section is to extend this result to all values of $\beta$ between 0 and 1. \\

Let us first introduce a scalar auxiliary function $g(y|\tilde{y};x)$ as follows:
\bal
\forall y,\tilde{y},x>0,\quad g(y|\tilde{y};x)=\conv{d}(x|y) + \conc{d}(x|\tilde{y}) + (y-\tilde{y})\conc{d}'(x|\tilde{y}) + \bar{d}(x)
\eal
where $\conv{d}(x|y)$, $\conc{d}(x|y)$ and $\bar{d}(x|y)$ are defined in Table~\ref{tab:beta}. By immediate application of Theorem~\ref{theo:auxfunc} to the scalar case, $g(y|\tilde{y};x)$ is an auxiliary function to $d(x|y)$. In particular, $g(\tilde{y}|\tilde{y};x)=d(x|\tilde{y})$.
Then, we have the following preliminary result.\\

\begin{mylemma} \label{lemma:Gproptog}
For all $\beta\in\mathbb{R}$,
\begin{equation}
G_k(h_k | \tih)=\frac1{\tilde{h}_k^{\beta-1}}\left(\sum_f w_{fk}\tilde{v}_f^{\beta-1}\right) g(h_k | \tilde{h}_k ; h_k^{\text{H}}) + \cst.
\label{eq:Gproptog}
\end{equation}
\end{mylemma}
\begin{proof}
For each of the four possible expressions of $(\conc{d},\conv{d})$ given in Table~\ref{tab:beta}, the validity of \eqref{eq:Gproptog} can be checked straightforwardly by direct verification.
\end{proof}

As already mentioned in Section~\ref{subsec:defaux}, the MM update \eqref{eqn:max} is not the only way of taking advantage of the auxiliary function $G(\hh|\tih)$ to obtain a decrease of $C(\hh)$: any update satisfying \eqref{eqn:idea} also ensures that $C(\hh)$ does not increase.
This is a key remark to understand the behavior of the heuristic algorithm for $\beta\in(0,1)$, given the following property.\\

\begin{mytheorem}\label{theo:heuristic}
For all $\beta\in(0,1)$, and all $\tih$ such that conditions (i)-(ii) of Theorem~\ref{theo:auxfunc} hold, the heuristic algorithm produces nonincreasing values of $C(\hh)$, according to the following inequality:
\begin{equation}
G(\hh^{\text{H}} | \tih)\leq G(\tih | \tih).
\label{eq:decrmaj}
\end{equation}
\end{mytheorem}

\begin{proof}
For all $\beta\in(0,1)$, straightforward calculations yield
\bal
g(\tilde{y}|\tilde{y};x)-g(x|\tilde{y};x)&=\conv{d}(x|\tilde{y})-\conv{d}(x|x)-(x-\tilde{y})\conc{d}'(x|\tilde{y})\\
&=\frac{1}{1-\beta}\,\tilde{y}^{\beta}(1-\beta+\beta\theta-\theta^{\beta})
\eal
where $\theta=x/\tilde{y}$. Since $f(\theta)=\theta^{\beta}$ is a concave function of $\theta$, we have $f(\theta)\leq f(1)+(\theta-1)f'(1)$, which also reads $\theta^{\beta}\leq1+(\theta-1)\beta$. Hence, $g(\tilde{y}|\tilde{y};x)-g(x|\tilde{y};x)\geq0$
for all $x$, $\tilde{y}$. The latter inequality implies $\forall k,\,g(h_k^{\text{H}}|\tilde{h}_k,h_k^{\text{H}})\leq g(\tilde{h}_k|\tilde{h}_k,h_k^{\text{H}})$, so that we have
$G_k(h_k^{\text{H}} | \tih)\leq G_k(\tilde{h}_k | \tih)$ according to \eqref{eq:Gproptog}, which leads to the result by summation over $k$.
\end{proof}

\cite{cao99} show that inequality \eqref{eq:decrmaj} becomes an equality in the case $\beta=0$, so that each heuristic update yields
$G(\hh^{\text{H}} | \tih)= G(\tih | \tih)$. In this particular case, the heuristic algorithm can be called a ``majorization-equalization'' algorithm, a class of algorithms described in next section. For values of $\beta$ outside the range $[0,2]$, inequality \eqref{eq:decrmaj} does not hold anymore.\footnote{\label{foot:reversed}Indeed, we can prove that the reversed inequality holds for all $\beta<0$, while no systematic result is known for $\beta>2$.} Of course, this does not mean that the heuristic updates produce increasing values of $C(\hh)$. On the contrary, numerical simulations tend to indicate that they always produce nonincreasing values of $C(\hh)$, but proving this is still an open issue for $\beta\not\in[0,2]$. Compared to MM updates, heuristic updates produce larger or equal steps for all $\beta$, since it can trivially be shown that
\begin{equation}
\forall k,\quad |h_k^{\text{H}}-\tilde{h}_k| \geq |h_k^\text{MM}-\tilde{h}_k|.
\label{taillepas}
\end{equation}

For $\beta\not\in[1,2]$, numerical simulations indicate that the heuristic algorithm is faster than the MM algorithm (and we recall that the two algorithms coincide for $\beta\in[1,2]$). Given \eqref{taillepas}, skipping from the latter to the former has an effect comparable to that of overrelaxation: on the average, stretching the steps allows to reduce their number to reach convergence. This will be discussed in more details in Section~\ref{sec:over}.\\

In order to produce even larger steps for $\beta\in[0,2]$, and yet nonincreasing values of $C(\hh)$, the following section explores the concept of majorization-equalization.

\subsection{Majorization-Equalization (ME) algorithm} \label{sec:me}

Let us introduce the general notion of ME update by the fact that the new iterate $\hh^{\text{ME}}$ fulfills
\begin{equation}
\label{eq:generalME}
G(\hh^{\text{ME}} | \tih)= G(\tih | \tih).
\end{equation}
Eq.~\eqref{eq:generalME} actually defines a level set rather than a single point. Let us concentrate on the following more constrained and manageable condition, given the separability of $G(\hh|\tih)$:
$$
\forall k,\quad G_k(h_k^{\text{ME}} | \tih)= G_k(\tilde{h}_k | \tih).
$$
Given \eqref{eq:Gproptog}, this amounts to solve the following equation for $y$, for any $\tilde{y},x>0$:
\begin{equation}
\label{eq:MEsurg}
g(y|\tilde{y};x)=g(\tilde{y}|\tilde{y};x).
\end{equation}
Since $g(y|\tilde{y};x)$ is strictly convex w.r.t $y$, \eqref{eq:MEsurg} has not more than two solutions, one of them being $\tilde{y}$. {By construction, the selection of the other solution (provided that it exists) will provide ME steps that are larger than MM updates, i.e.,}
\begin{equation}
\forall k,\quad |h_k^{\text{ME}}-\tilde{h}_k| \geq |h_k^\text{MM}-\tilde{h}_k|,
\label{taillepasME}
\end{equation}
{as illustrated by Figure~\ref{fig:ME}. To go further on the determination of this solution,} a case-by-case analysis must be performed, depending on the range of $\beta$. 

\paragraph{Case 1: $\beta\in[0,1)$} In that case we have
\begin{equation}
g(y|\tilde{y};x)=\frac{1}{1-\beta} x \,y^{\beta -1}+y\tilde{y}^{\beta-1}+\cst.
\label{eq:gbetaleun}
\end{equation}
Let us remark that
\begin{equation}
\forall\,\tilde{y}, x>0,\quad\lim_{y\rightarrow0} g(y|\tilde{y};x)=\lim_{y\rightarrow\infty} g(y|\tilde{y};x)=\infty,
\label{eq:limites}
\end{equation}
so that \eqref{eq:MEsurg} always admits two positive solutions (or one double positive solution if $\tilde{y}=x$), one of the two being $y=\tilde{y}$. The other one is the solution of interest. However, it is not closed-form, except for specific values of $\beta$ (see Table~\ref{tab:betapoly}). More precisely, when $\beta=1-1/d$ and $d$ is an integer, the solution can be found by solving the following polynomial equation of degree $d$, for $z=y^{1/d}$:
\begin{equation}
(1-\beta)\,\sum_{\ell=1}^d\tilde{z}^{d-\ell}z^\ell-x=0
\label{eq:poly}
\end{equation}
where $\tilde{z}=\tilde{y}^{1/d}$. Not surprisingly, the simplest case $\beta=0$ ($d=1$) leads us to $y=x$, and thus to $h_k^{\text{ME}}=h_k^{\text{H}}$. The case $\beta=0.5$ ($d=2$) is more interesting. The extraction of the positive root of \eqref{eq:poly} then provides the following update formula:
\bal \label{eqn:ME05}
h_k^{\text{ME}}=\frac{\tilde{h}_k}4\left(\sqrt{1+8\frac{h_k^{\text{H}}}{\tilde{h}_k}}-1\right)^2.
\eal
Let us remark that this expression does not correspond to a multiplicative update, although it ensures that positivity is maintained.\\
\begin{table}[t]
\begin{center}
$\begin{array}{|c|c|c|c||c|}
\hline
~~~\beta \leq0~~~&  0\leq\beta\leq1 & 1 \le \beta \le 2 & ~~\beta\geq2~~ &~~d~~\\\hline
 0 & 0 & 2& 2&1\\\hline
-1&1/2&3/2&3 &2\\\hline
-2&2/3&4/3&4 &3\\\hline
-3&3/4&5/4&5 &4\\\hline
\end{array}
$\end{center}
\caption{Values of $\beta$ for which ME updates are closed-form, by root extraction of polynomials of degree $d$.}
\label{tab:betapoly}
\end{table}

\paragraph{Case 2: $\beta\in(1,2]$} In that case we have
\bal
g(y|\tilde{y};x)=\frac{1}{\beta}\,y^{\beta}-\frac1{\beta-1} xy^{\beta-1}+\cst.
\eal
$g(y|\tilde{y};x)$ tends toward $\infty$ for $y\rightarrow\infty$, but it remains finite for $y\rightarrow0$. As a consequence, Eq.~\eqref{eq:MEsurg} only admits the trivial solution $y=\tilde{y}$ if $g(\tilde{y}|\tilde{y};x)>g(0|\tilde{y};x)$, and also the unwanted solution $0$ if $g(\tilde{y}|\tilde{y};x)=g(0|\tilde{y};x)$. It is only when $g(\tilde{y}|\tilde{y};x)<g(0|\tilde{y};x)$ that a positive, non trivial solution exists. This solution is closed-form for specific values of $\beta$ given in Table~\ref{tab:betapoly}. They correspond to $\beta=1+1/d$, where $d$ is an integer. Eq.~\eqref{eq:MEsurg} then amounts to solve the following polynomial equation of degree $d$, for $z=y^{1/d}$:
\bal
\sum_{\ell=0}^d\tilde{z}^{d-\ell}z^\ell-(d+1)\,x=0,
\eal
with $\tilde{z}=\tilde{y}^{1/d}$. The simplest case is $\beta=2$ ($d=1$), and the solution is then given by $y=2x-\tilde{y}$ if $\tilde{y}<2x$, which yields the overrelaxed update
\bal \label{eqn:ME2}
h_k^{\text{ME}}=2h_k^{\text{MM}}-\tilde{h}_k,
\eal
provided that $\tilde{h}_k<2h_k^{\text{MM}}$. Note that this result more simply stems from the fact that when $\beta=2$ the auxiliary function is parabolic and thus symmetric with respect to $h_k^{\text{MM}}$. In the case $\beta=1.5$ ($d=2$), a positive ME update exists if $\tilde{h}_k<3h_k^{\text{MM}}$, and it takes the following form:
\bal \label{eqn:ME15}
h_k^{\text{ME}}=\frac{\tilde{h}_k}4\left(\sqrt{12\frac{h_k^{\text{MM}}}{\tilde{h}_k}-3}-1\right)^2.
\eal
As we need an update strategy that is defined everywhere, we propose to rely on a linear mixture between the MM update and a prolonged version of ME, defined as
\bal
h_k^\theta=\theta h_k^{\text{pME}}+(1-\theta)h_k^{\text{MM}}
\label{eqn:pME}
\eal
where $\theta\in(0,1)$ and $h_k^{\text{pME}}$ prolongs the ME update by zero when the latter does not exist:
\bal
h_k^{\text{pME}}=
\begin{cases}
h_k^{\text{ME}}&\text{if }h_k^{\text{ME}}\text{ is defined}\\
0&\text{otherwise}
\end{cases}
\eal
It is mathematically easy to check that $h_k^\theta$ fulfills Eq.~\eqref{eqn:idea} for all $\theta\in[0,1]$, and that positivity is maintained for all $\theta\in[0,1)$. In practice, values of $\theta$ near one may be favored to produce larger steps. \\

\begin{figure}[t]
\begin{center}
\begin{tabular}{ccc}
(a) $\beta=0.5$ & (b) $\beta=1.5$ & (c) $\beta=2$ \\
\epsfig{figure=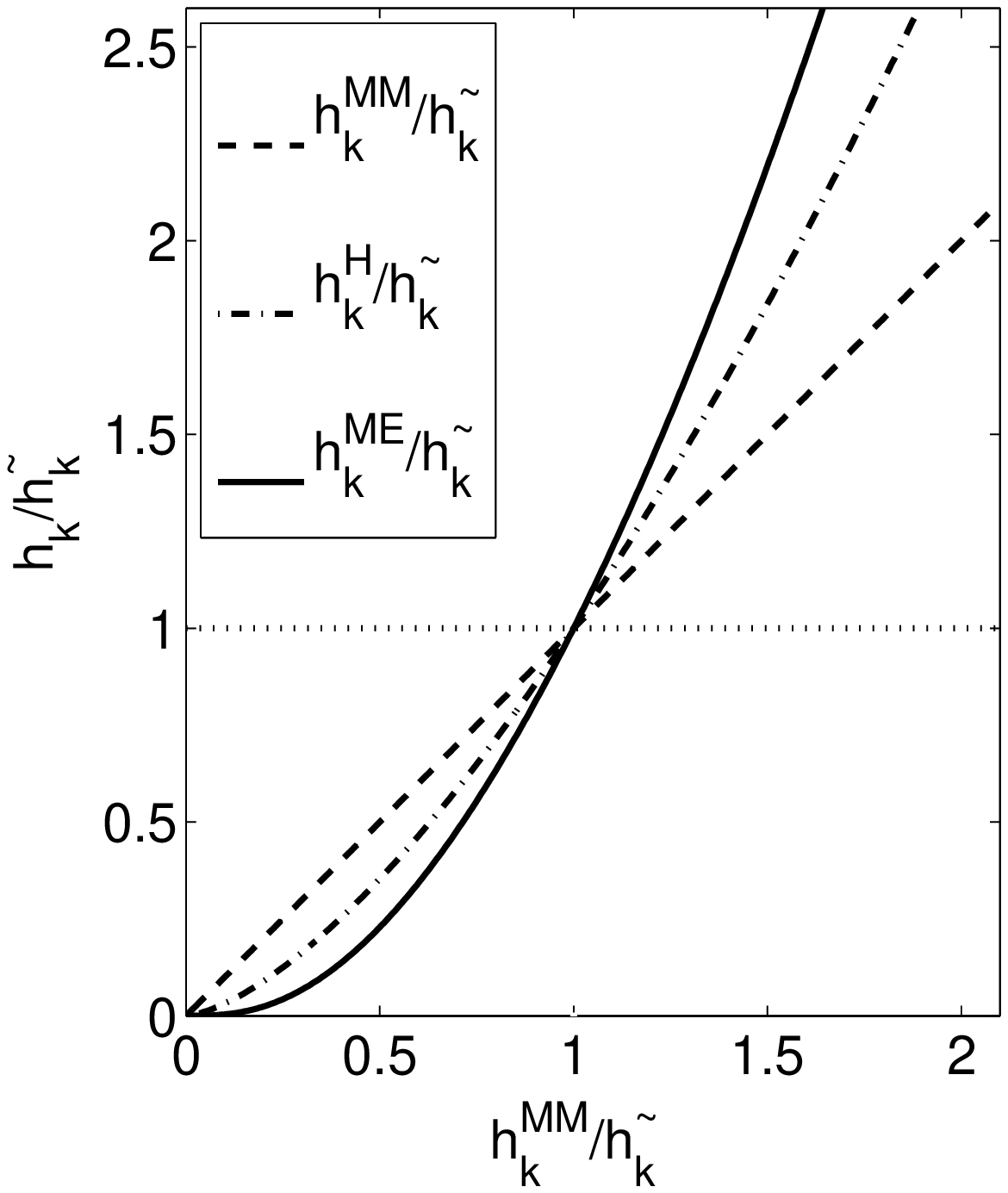,width=.3\linewidth} & \epsfig{figure=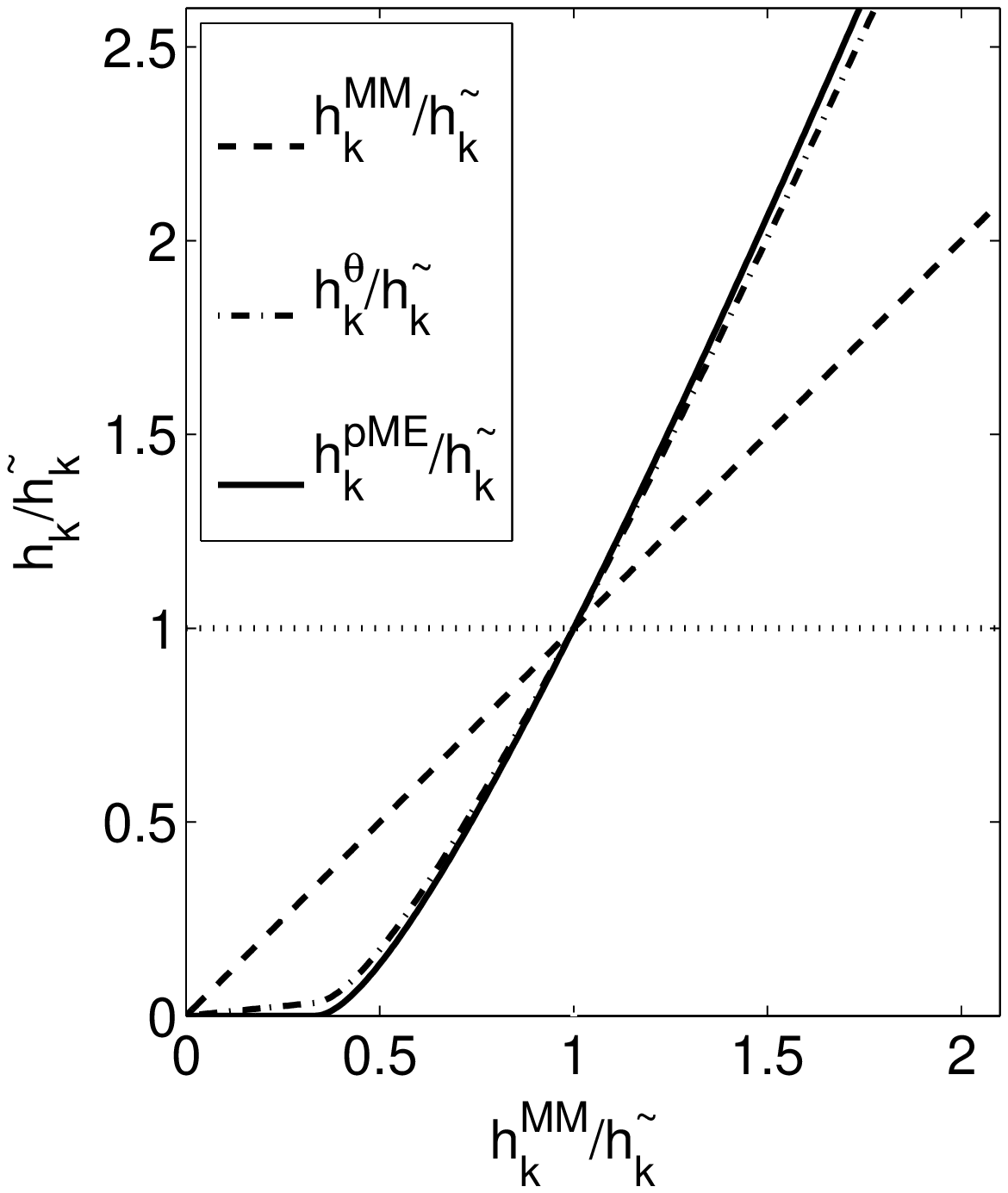,width=.3\linewidth} & \epsfig{figure=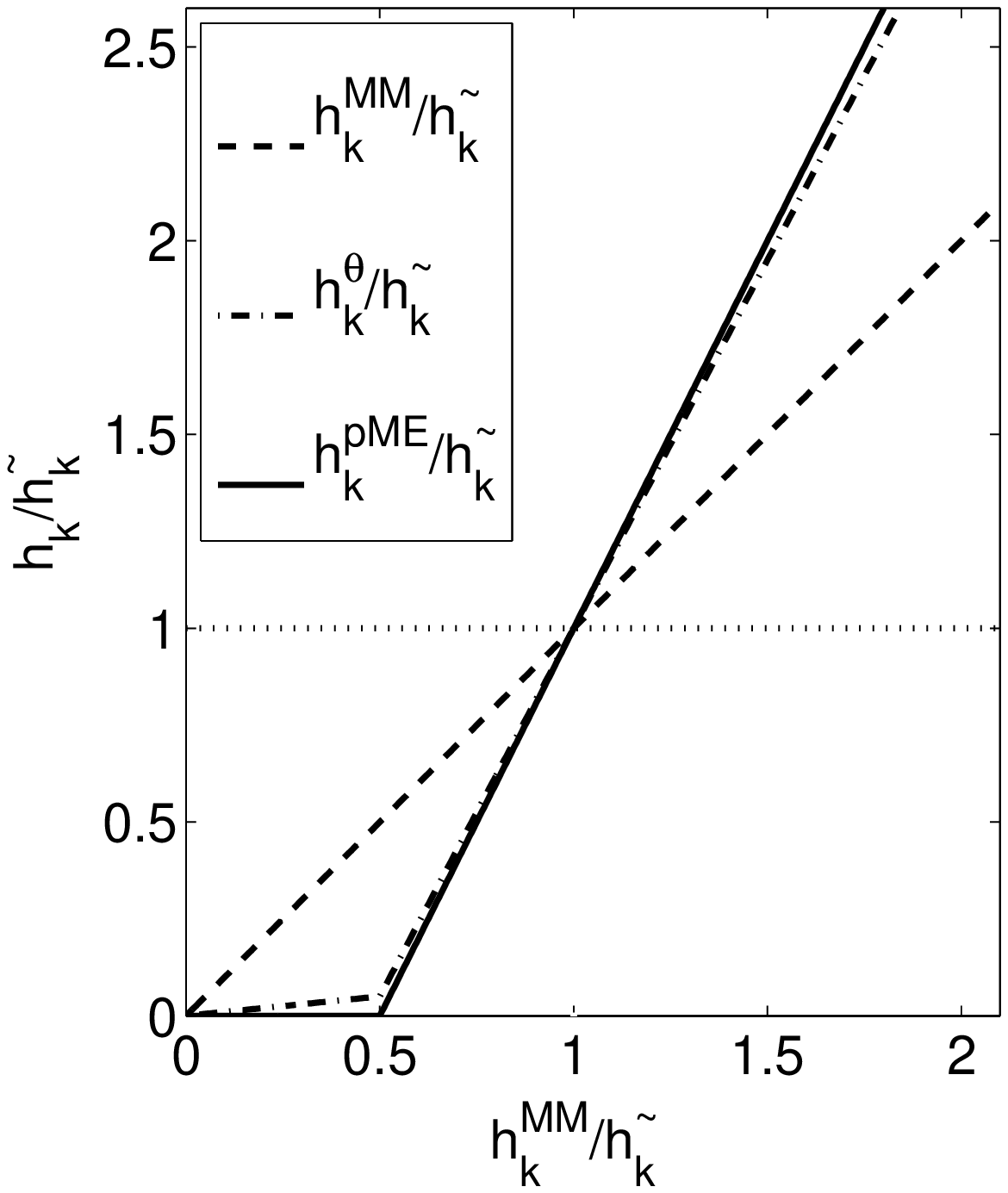,width=.3\linewidth}
\end{tabular}
\end{center}
\caption{Normalized updates $h_k/\tilde{h}_k$ as functions of $h_k^{\text{MM}}/\tilde{h}_k$ ($\theta = 0.9$). The region between the dotted, horizontal line and the solid line correspond to the steps that fulfill Eq.~\eqref{eqn:idea}. The larger departure from the horizontal line, the larger the step. }
\label{fig:steplengths}
\end{figure}

When $\beta<0$ or $\beta>2$, similar analyses can be conducted. In particular, there are specific values of $\beta$ for which a closed-form expression of ME updates is available according to Table~\ref{tab:betapoly}.\\

When $\beta<0$, ME updates always exist since \eqref{eq:gbetaleun} and \eqref{eq:limites} still hold. Moreover, they provide nonincreasing values of $C(\hh)$, while the latter monotonicity property is not yet proved for the heuristic algorithm. However, simulations tend to indicate that the heuristic algorithm is faster than the ME algorithm (which is itself faster than the MM algorithm) in the case $\beta<0$. This is in conformity with the fact that ME steps can then be proved to be smaller than heuristic steps {(on the basis of the reversed inequality mentioned in Footnote~\ref{foot:reversed}).}\\

When $\beta>2$, ME updates do not necessarily exist, akin to the case $\beta\in(1,2]$. When they exist, they provide nonincreasing values of $C(\hh)$, while the latter is not yet proved for the heuristic formula. However, since this range of $\beta$ values is not of utter practical interest, we will not go further into a detailed analysis here.

\subsection{Overrelaxation properties of the heuristic and ME updates} \label{sec:over}

As already stated, the heuristic and ME updates produce larger steps than the MM update, i.e., $|h_k^{\text{H}}-\tilde{h}_k|$ and $|h_k^{\text{ME}}-\tilde{h}_k|$ are larger than $|h_k^{\text{MM}}-\tilde{h}_k|$, for all values of $\beta \in \mathbb{R}$. This is a form of overrelaxation, which will be shown in Section~\ref{sec:simus} to produce faster convergence in practice. The normalized ME (or pME) and heuristic updates studied in the above sections can be written as a function of $h_k^{\text{MM}} / \tilde{h}_k$, such that
\bal
\frac{{h}_k}{ \tilde{h}_k} = f\left(\frac{h_k^{\text{MM}}}{ \tilde{h}_k}\right)
\eal
where ${h}_k$ is either ${h}_k^{\text{H}}$, ${h}_k^{\text{ME}}$, ${h}_k^{\text{pME}}$ or ${h}_k^{\theta}$. For the heuristic update, the function is simply given by $f(x) = x^{1/\gamma(\beta)}$. For $\beta \in \{0.5, 1.5, 2\}$, the function $f$ corresponding to the ME/pME update is easily derived from equations~\eqref{eqn:ME05},~\eqref{eqn:ME15} and~\eqref{eqn:ME2}. Figure~\ref{fig:steplengths} displays the latter functions for the updates studied in this work when $\beta \in \{ 0.5, 1.5, 2 \}$. Overrelaxation appears from the fact that  ${h}_k < {h}^{\text{MM}}_k$ whenever ${h}^{\text{MM}}_k < \tilde{h}_k$ (steps towards left) and ${h}_k > {h}^{\text{MM}}_k$ whenever ${h}^{\text{MM}}_k > \tilde{h}_k$ (steps towards right).

General results about overrelaxation of MM algorithms are given by \cite{sala03}, and in particular in the case of NMF. The authors consider the specific case of KL divergence but their study holds for any divergence. They show that, \emph{in a neighborhood of a stationary point}, {for any $\eta \in (0,2)$}, relaxed updates $h_k^{\text{R}}$ of the form
\bal
\frac{h_k^{\text{R}}}{\tilde{h}_k} = \left( \frac{h_k^{\text{MM}}}{\tilde{h}_k} \right)^\eta
\eal
will converge to the same point than $h_k^{\text{MM}}$, with a different, possibly better, rate of convergence. In particular, the optimal learning rate $\eta$, providing the largest rate of convergence, can be computed from the eigenvalues of the Jacobian, at convergence, of the mapping that relates $h_k^{\text{MM}}$ at iteration $(i)$ to $h_k^{\text{MM}}$ at iteration $(i+1)$. The optimal learning rate is shown to be always greater or equal to 1. A similar result was recently obtained by \cite{bad10}. However, these results do not translate into a practical algorithm, because the latter relaxation property only holds locally, and the computation of the optimal learning rate
requires the stationary point to be known. As such, \cite{sala03} propose an adaptive scheme which incrementally proposes values of $\eta$ greater than one at each iteration, and backtrack to $\eta =1$ when the criterion ceases decreasing.

Our results show that for $\beta \in (0,1)$ the learning rate $\eta = 1/\gamma(\beta) = 2 - \beta $, corresponding to the heuristic update, ensures descent of the criterion everywhere. The results of \cite{sala03} indicate that the learning rate can be increased to  $\eta =2$ when the algorithm approaches the solution. Note that in the neighborhood of the solution the Taylor approximation $f(x) \approx f(1) + f'(1) (x-1)$ applied to $f(x) = x^\eta$ implies that
\bal
({h_k^{\text{R}}} - {\tilde{h}_k}) \approx \eta ( h_k^{\text{MM}} - \tilde{h}_k ).
\eal
A similar approximation carried out with the ME/pME updates defined by equations~\eqref{eqn:ME05},~\eqref{eqn:ME15} for $\beta \in \{0.5, 1.5\}$ reveals that in these two cases $f'(1) = 2$ (and by construction $f(1)=1$), so that, in a neighborhood of the solution we have
\bal
( h_k^{\text{ME}} -  \tilde{h}_k) \approx 2 ( h_k^{\text{MM}} - \tilde{h}_k ).
\eal
This means that the ME algorithms produce the largest admissible learning rate $\eta =2$ in the neighborhood of the solution, while avoiding to adapt the learning rate so as to ensure monotonicity of the criterion. This results holds everywhere for $\beta =2$, see \eqref{eqn:ME2}, by symmetry of the auxiliary function w.r.t to $h_k^{\text{MM}}$. The interested reader may also refer to \citep{lan01,eusipco10} for relaxation of multiplicative algorithms using adaptative learning rates computed through line search.

\subsection{Implementation and complexity of the algorithms}

As seen in previous section, the update rules of all the studied algorithms can be expressed as functions of the ratio $\nabla_{h_k}^- C(\tih) / \nabla_{h_k}^+ C(\tih) $, which dominates the algorithmic complexities. Fortunately, the latter ratio takes a simple matrix form that leads to efficient implementations. As such, getting back to the original factorization problem, the heuristic update~\eqref{eq:heur_update} for factors $\H$ and $\W$ can conveniently be expressed in the following matrix form
\bal
\H \leftarrow& \H . \frac{ \W^T \, [ (\W \H)^{.(\beta-2)} . \V ] }{\W^T \, [\W \H]^{.(\beta-1)}} \label{eqn:Hbeta} \\
\W \leftarrow& \W . \frac{  [ (\W \H)^{.(\beta-2)} . \V ] \, \H^T }{ [\W \H]^{.(\beta-1)} \, \H^T } \label{eqn:Wbeta}
\eal
where the division ${\cdot}/{\cdot}$ is here taken entrywise. The MM update simply involves bringing the corrective ratio to the power $\gamma(\beta)$, and the ME update involves applying a function specific to the value of $\beta$. Hence, the algorithms have similar complexity $\mathcal{O}(FKN)$ and their implementation take simple forms. MATLAB implementations of the algorithms discussed in this paper are available online at \url{http://perso.telecom-paristech.fr/~fevotte/Code/code_beta_nmf.zip}

\section{{Simulations}} \label{sec:simus}

In this section we report performance results of $\beta$-NMF algorithms for the specific values $\beta \in \{ 0.5, 1.5, 2\}$. These values are chosen for their practical interest and because a simple ME algorithms exist in their case. As such this section will evidence the performance improvement brought by the ME approach over the MM or heuristic approaches, with similar computational burden. More precisely, the ME algorithm considered in this section is the mixture of prolonged ME and MM, defined by Eq.~\eqref{eqn:pME} and with $\theta = 0.95$, but we will still refer to it as ME for simplicity. The algorithms for all three considered values of $\beta$ are compared on small-sized synthetic data in Section~\ref{sec:synth}. The algorithms for $\beta = 0.5$ are analyzed in Section~\ref{sec:audio} on the basis of a small music transcription example as this specific value of $\beta$ has proven efficient for this task \citep{fitz09,vinc10,henn10}.

In the following results we will display the cost values through iterations as well as, following \cite{gonz05}, ``KKT residuals''. The residuals allow to monitor convergence to a stationary point and are here defined as
\bal
\mathrm{KKT}(\W) &= \| \min \left\{ \W,  [ (\W \H)^{.(\beta-2)} . (\W \H - \V) ] \H^T  \right\}  \|_1 / FK \\
\mathrm{KKT}(\H) &= \| \min \left\{ \H, \W^T  [ (\W \H)^{.(\beta-2)} . (\W \H - \V) ] \right\} \|_1 / KN.
\eal
They are meant to converge to zero, by Eq.~\eqref{eqn:kkt}. Again, the monotonicity of the heuristic, MM and ME algorithms does not imply convergence of the iterates to a stationary point. Hence, displaying the KKT residuals allows to experimentally check whether convergence is achieved in practice. 

One iteration of each algorithm consists of updating $\W$ given $\H^{(i-1)}$ and $\H$ given $\W^{(i)}$, and then normalize $\W^{(i)}$ and $\H^{(i)}$ to eliminate trivial scale indeterminacies that leave the cost function unchanged. The normalization step consists of rescaling each column of $\W$ so that $\| \ve{w}_k \|_1 = 1$ and rescale the $k^{th}$ row of $\H$ accordingly. The normalization step is not required \textit{per se} but is useful to display and compare the KKT residuals, which are scale-sensitive.

\subsection{Factorization of synthetic data} \label{sec:synth}

We consider a synthetic data matrix $\ve{V}$ constructed as $\ve{V} = \ve{W}^* \ve{H}^*$ where the ground truth factors are generated as the absolute values of Gaussian noise.\footnote{E.g., in MATLAB notation \texttt{V = abs(randn(F,K))*abs(randn(K,N))}. } The matrix can be exactly factorized so that all algorithms should converge to a solution such that $D(\V | \W \H) = 0$. The dimensions are $F=10$, $N=25$, $K=5$. The algorithms (heuristic, MM, ME for $\beta = 0.5$, MM and ME for $\beta \in \{1.5, 2\}$) are run for $10^5$ iterations and initialized with positive random values. Fig.~\ref{fig:synth1},~\ref{fig:synth2} and~\ref{fig:synth3} display for each of the 3 values of $\beta$ the normalized cost values $D(\V| \W \H)/FN$, the KKT residuals, as well as ``fit residuals" computed as $\| \W^{(i)} - \hat{\W} \|_F / FK$ and $\| \H^{(i)} - \hat{\H} \|_F / KN$, where $\hat{\W}$ and $\hat{\H}$ are the factor estimates at the end of the $10^5$ iterations and $\| . \|_F$ is the Frobenius norm. The fit residuals allow to measure the closeness of the current iterates to their end value.\\

The cost values in all three cases converge to zero as an exact factorization is reached (oscillations appear in the end iterations as machine precision is reached). Convergence is achieved in all three cases, as shown by both the cost values and KKT residuals. We visually inspected the factorizations returned by the algorithms. For each value of  $\beta \in \{ 0.5, 1.5\}$, the different algorithms appeared to converge to the same solution (and solutions obtained for the two values of $\beta$ appeared comparable). This was less clear for $\beta = 2$, where ME appeared to reach out a different solution than MM. Still, in this run ME provides fastest convergence for every considered value of $\beta$. Other runs, obtained from other starting points (obtained randomly), tend to show that when the compared algorithms converge to the same solution, then ME converges faster. Convergence to a common solution can be controlled in the specific case where $\W$ is fixed and $\beta \in \{1.5, 2\}$, because the objective function is then convex w.r.t $\H$. In this scenario ME was found to always converge faster than MM. These simulations are reported in the companion report available online at \url{http://perso.telecom-paristech.fr/~fevotte/Samples/neco11/beta_nmf_supp.pdf} \\

The fit residuals in Fig.~\ref{fig:synth1},~\ref{fig:synth2} and~\ref{fig:synth3} show that full convergence will not need be attained to obtain satisfying solutions for most applications as the fit residual will be considered sufficiently small after a few hundred iterations. Note that the factor iterates do not necessarily converge to the ground truth values $\W^*$ and $\H^*$ (and this is what we observed indeed) because of the identifiability ambiguities inherent to NMF \citep{don04,lau08}.\\

Using a MATLAB implementation run on Mac 2.6 GHz with 2 GB RAM, the CPU time required by each algorithm for the $10^5$ iterations is about 60~s for $\beta \in \{0.5, 1.5\}$ and 20~s for $\beta = 2$, including the computation of the cost values and KKT residuals. The ME algorithm is marginally more expensive than MM, itself only slightly more expensive than the heuristic algorithm, for $\beta = 0.5$. The CPU times incurred by the algorithms when $\beta = 2$ is considerably lower thanks to simplifications in Eq.~\eqref{eqn:Hbeta} and~\eqref{eqn:Wbeta}. Indeed, in latter case the term $(\W \H) \H^T$ appearing at the denominator can more efficiently be computed as $\W (\H \H^T)$, which involves a multiplication of matrices with smaller sizes.

\begin{figure}[t]
 \begin{center}
   \includegraphics[width=0.85\linewidth]{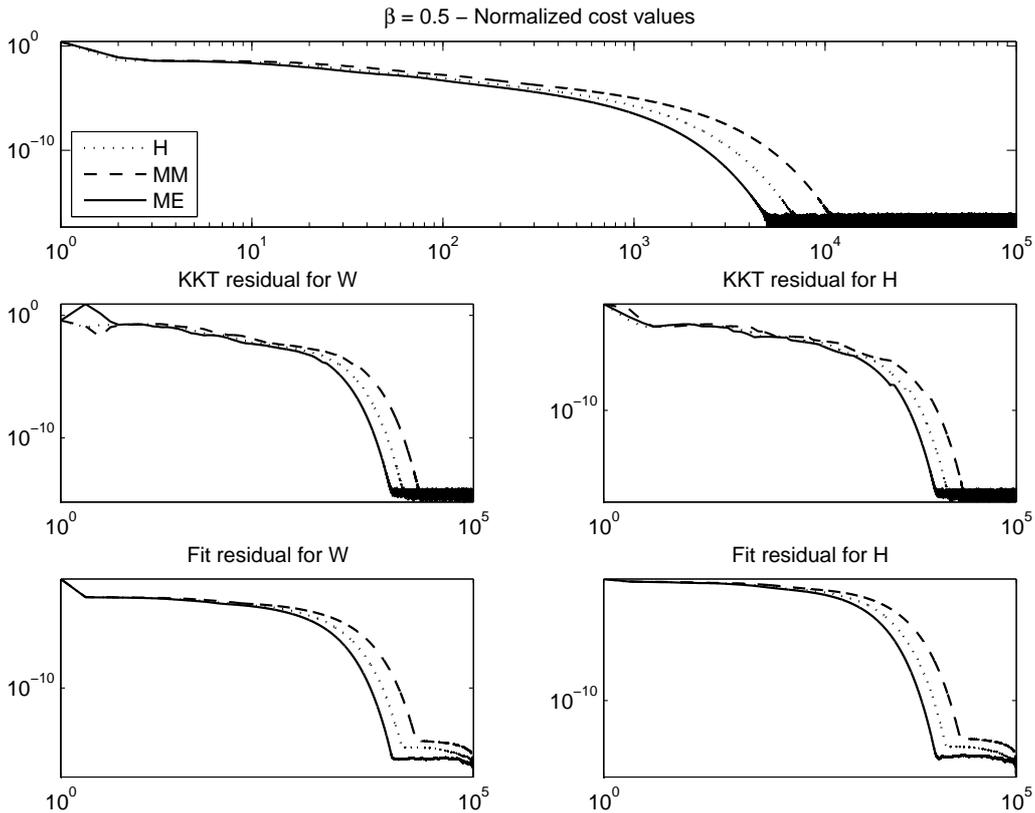}\\   
     \caption{One run of the heuristic (H), ME and MM algorithms on synthetic data with $\beta = 0.5$. Logarithmic scales for both x- and y- axes.}
   \label{fig:synth1}
   \end{center}
\end{figure}

\begin{figure}[t]
 \begin{center}
   \includegraphics[width=0.85\linewidth]{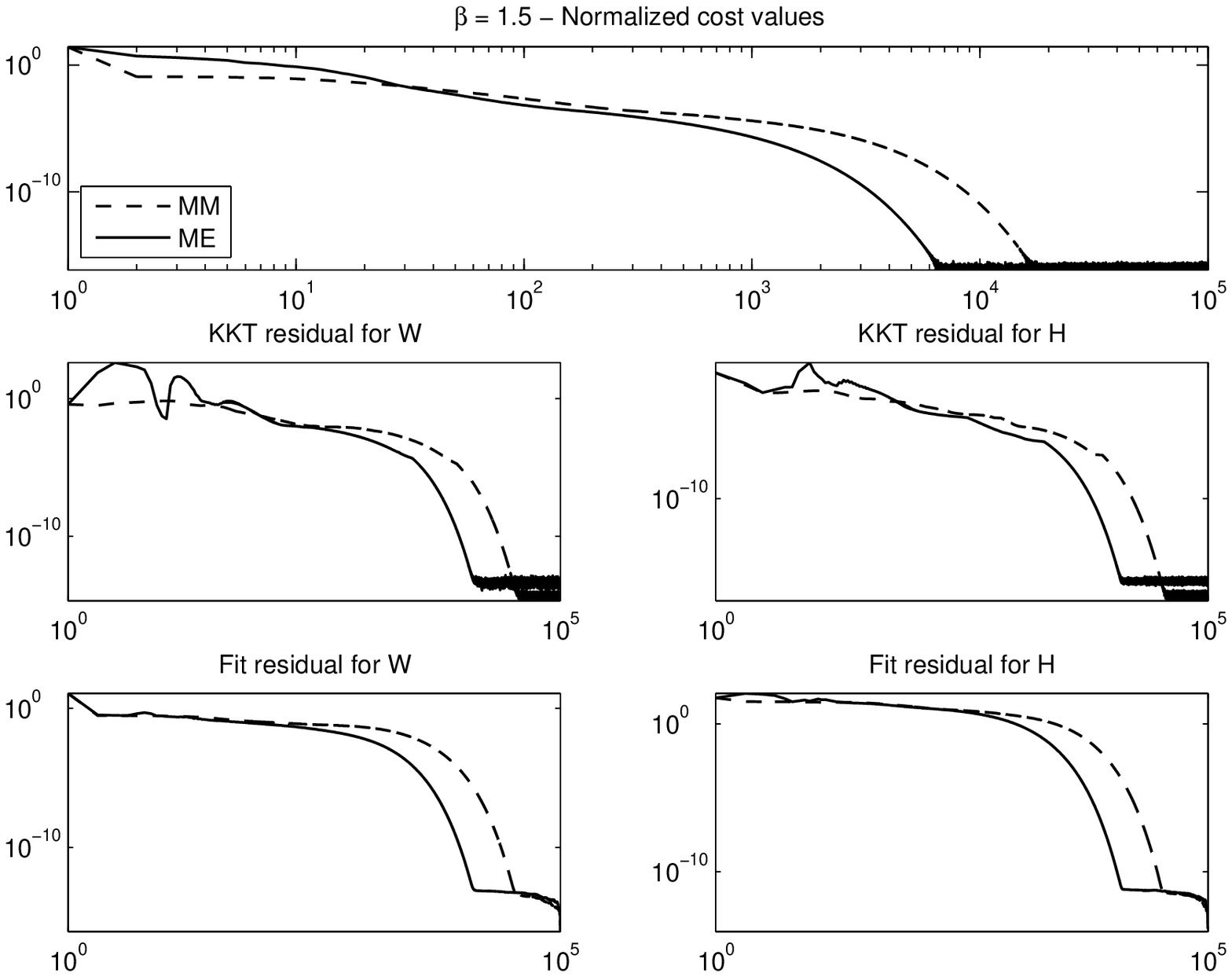}\\   
     \caption{One run of the ME and MM algorithms on synthetic data with $\beta = 1.5$. Logarithmic scales for both x- and y- axes.}
   \label{fig:synth2}
   \end{center}
\end{figure}

\begin{figure}[t]
 \begin{center}
   \includegraphics[width=0.85\linewidth]{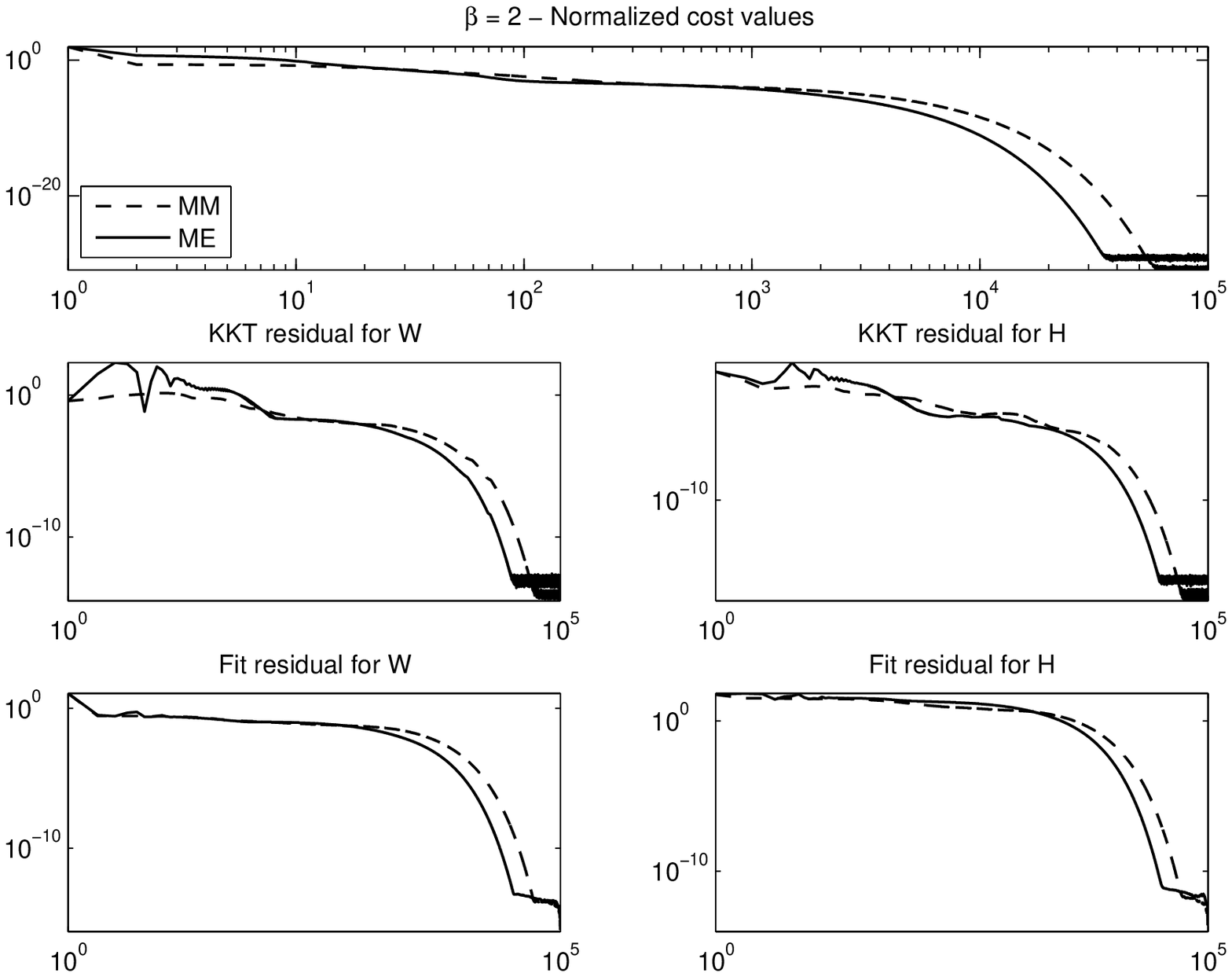}\\   
     \caption{One run of the ME and MM algorithms on synthetic data with $\beta = 2$. Logarithmic scales for both x- and y- axes.}
   \label{fig:synth3}
   \end{center}
\end{figure}

\subsection{Audio spectrogram decompostion} \label{sec:audio}

This section addresses the comparison of the heuristic, MM and ME algorithms for $\beta = 0.5$ applied to an audio spectrogram. We consider the short piano sequence of \citep{neco09}, recorded in live conditions, composed of 4 musical notes, played all at once in the first measure and then played by pairs in all possible combinations in the subsequent measures. A magnitude spectrogram of the audio signal is computed, leading to nonnegative matrix data $\V$ of size $F=513$ frequency bins by $N=674$ time frames. The data is represented top-left of Fig.~\ref{fig:audio}.
	
As discussed in \citep{neco09}, $K$ was set to 6 so as to retrieve in $\W$ the individual spectra of each of the 4 notes and supplementary spectra corresponding to transients and residual noise. The three algorithms were initialized with common positive random values and run for $10^5$ iterations. Figure~\ref{fig:audio} displays the cost values and KKT residuals along the $10^5$ iterations. It was manually checked that the algorithms converged to the desired ``ground-truth" solution, i.e., the notes, transients and residual noise spectra are correctly unmingled. The three plots show that the ME provides fastest convergence overall though, judging from the KKT residuals, it appears that convergence is not achieved within the $10^5$ iterations. However, the musical pitch values (computed from $\W$ at every iteration) converge to their ground truth values after only 30, 50 and 580 iterations for ME, heuristic and MM, respectively. Other initializations yielded two types of results. In a minority of cases, either the heuristic and MM update, on one side, or the ME update, on the other side, converged to a local solution. In the large majority of cases the three algorithms converge to the same solution and the results are similar to those of Figure~\ref{fig:audio}: the heuristic algorithm produces largest decreases of the objective function in the early iterations and is then supplanted by ME. In some runs the pitch values converged faster with the heuristic algorithm than with ME, and it was found that MM is generally slower than the other two algorithms. These results suggest a mixed update of the form of Eq.~\eqref{eqn:pME} where the mixture parameter $\theta$ could be made iteration-dependent so as to give more weight to the heuristic update in the early iterations and then to ME.

\begin{figure}[t]
 \begin{center}
   \includegraphics[width=0.85\linewidth]{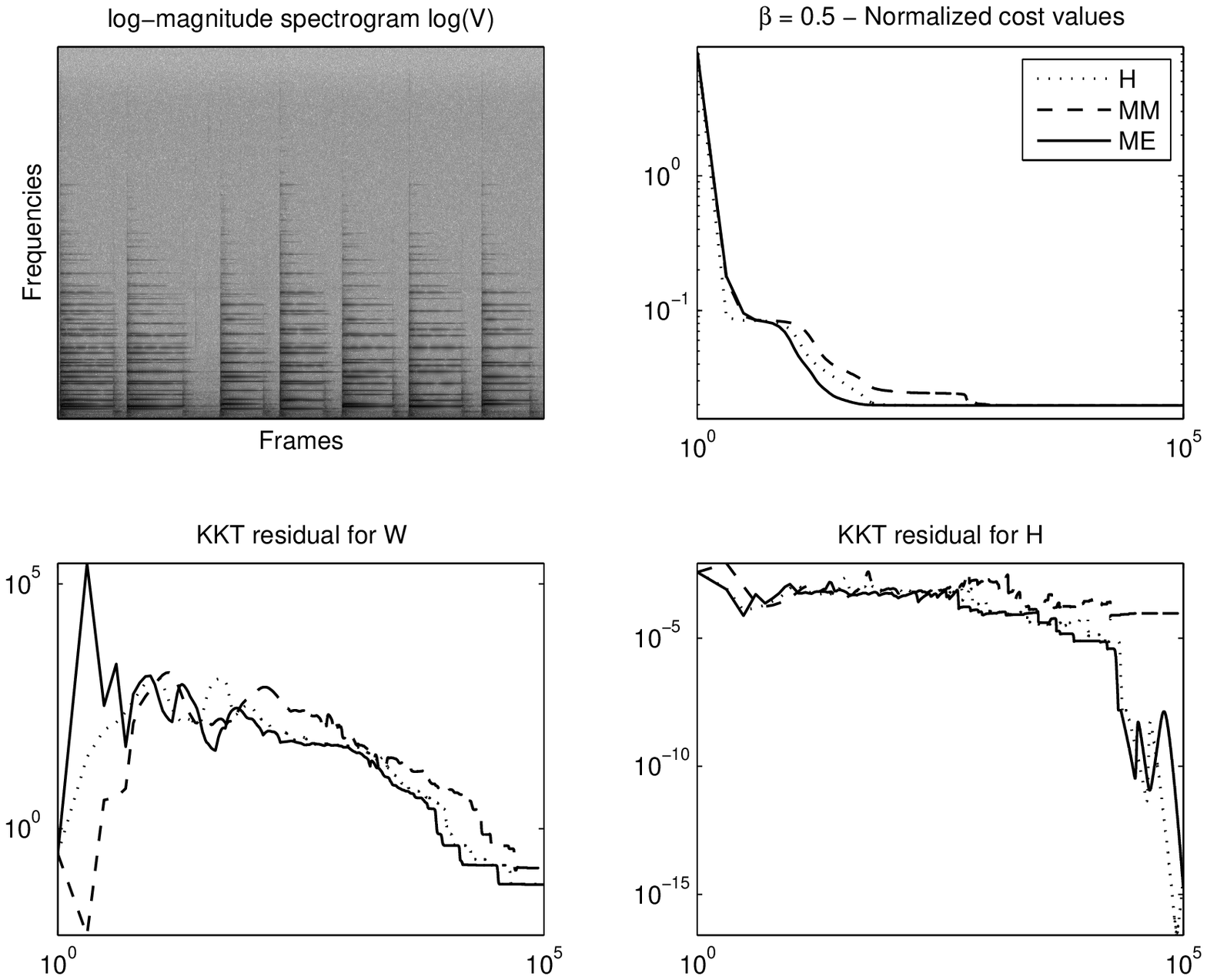}    
   \caption{One run of the heuristic (H), MM and ME algorithms on the piano magnitude spectrogram with $\beta = 0.5$. Logarithmic scales for both x- and y- axes.}
   \label{fig:audio}
   \end{center}
\end{figure}

\subsection{Face data decompostion} \label{sec:face}

Finally, in this section we consider decomposition of face data using $\beta$-NMF. We use the Olivetti dataset, composed of 10 grayscale 8 bits $64 \times 64$ face images of 40 people. We retrieved the data in MATLAB format from \url{http://cs.nyu.edu/~roweis/data.html}. The images are vectorized and form the columns of $\V$, with dimensions $F=4096$ and $N=400$. Fig.~\ref{fig:olivetti} displays the objective functions of one run of the ME and MM algorithms for $\beta \in \{1.5, 2\}$, and illustrate the faster convergence of ME. Other runs led to sensibly similar plots. 

\begin{figure}[t]
 \begin{center}
   \includegraphics[width=0.85\linewidth]{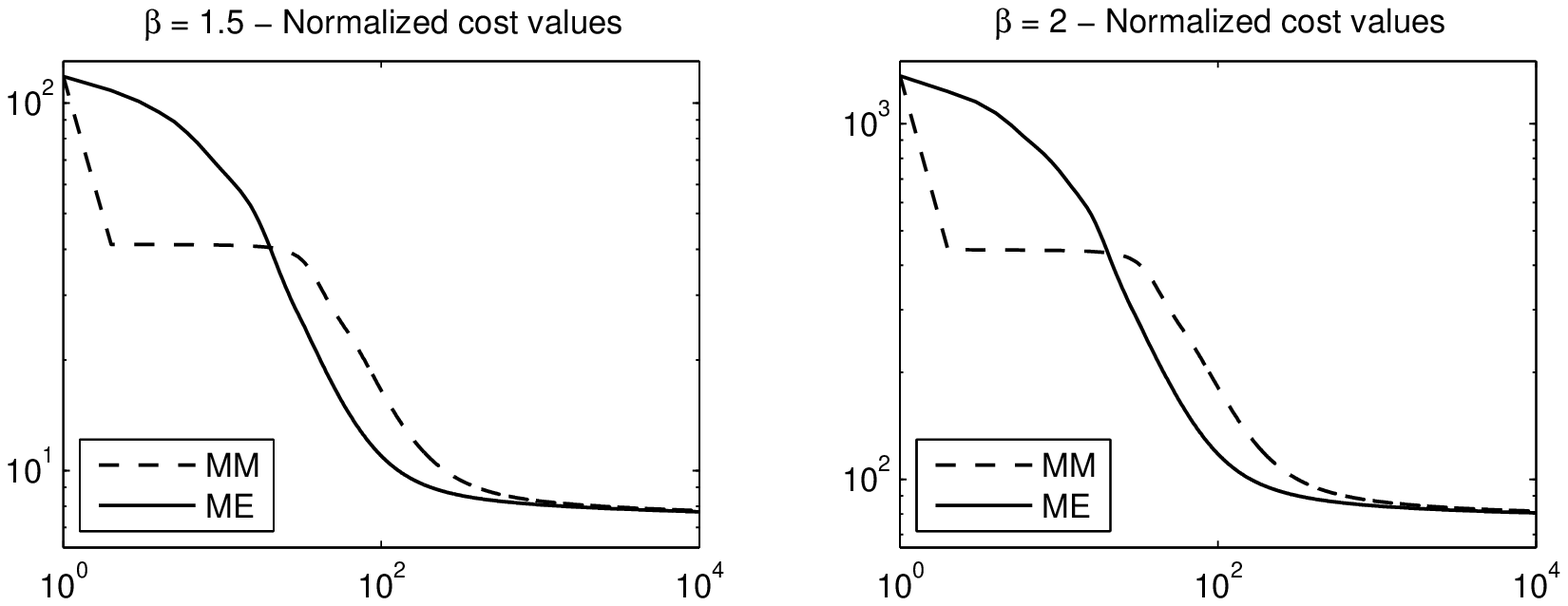}    
   \caption{One run of the MM and ME algorithms on the Olivetti dataset with $\beta = 1.5$ (left) and $\beta = 2$ (right). Logarithmic scales for both x- and y- axes.}
   \label{fig:olivetti}
   \end{center}
\end{figure}

As stated in the introduction, $\beta$-NMF is popular in audio signal processing where the value of $\beta$ can be controlled so as to improve transcription or separation accuracy. The idea of tuning the value of $\beta$ so as to optimize performance applies to any NMF-based method for any type of data. As such, to motivate the use of $\beta$-NMF in a non-audio setting we propose an image interpolation exemple, inspired by \citep{cic08b,cem09}, where we show the influence of $\beta$ on the reconstruction of missing data. We discard 25 \% of the Olivetti data randomly and produce NMF decompositions using the available data for $\beta \in \{-1,0,1,2,3\}$ and $K \in \{50, 100, 200 \}$. Accounting for the missing data requires minor modifications in the algorithms, basically multiplying  $\V$ and its approximate $\W \H$ with a binary mask in which zeroes indicate missing pixels, see \citep{ho08,cic08b,lero08,cem09,sma10} for similar setups. For simplicity we only considered the MM algorithm, as it is consistently defined for all values of $\beta$. It was run from 5 different initializations for every combination $(K, \beta)$, and the factorization yielding lowest end cost value was selected. Fig.~\ref{fig:interp} displays the original image, missing pixels and reconstructions for two of the images in the dataset. We have also computed the Peak Signal to Noise Ratio (PSNR) between original and reconstructed images.\footnote{PSNR is a standard evaluation criterion in image reconstruction, defined as $20 \log_{10} ( F P / \| \ve{v} - \hat{\ve{v}} \|_2 )$, where $\ve{v}$ and $\hat{\ve{v}}$ denote the vectorized original and reconstructed images, and $P$ is the maximum pixel possible value ($P=255$ in our case).} The maximum mean PSNR value (averaged over all 400 images) is obtained for $\beta =2$ and $K = 200$. However, since the PSNR value is equivalent to the Euclidean distance between the original and reconstructed image, it is expected that the optimal value of $\beta$ is biased towards the metric used to assess the quality of reconstruction. Perceptually, we often found the reconstruction obtained with $\beta =3$ to be more satisfying than with $\beta =2$.

\begin{figure}[t]
\begin{center}
 \begin{minipage}{0.45\linewidth}
  \begin{center}
   \includegraphics[width=\linewidth]{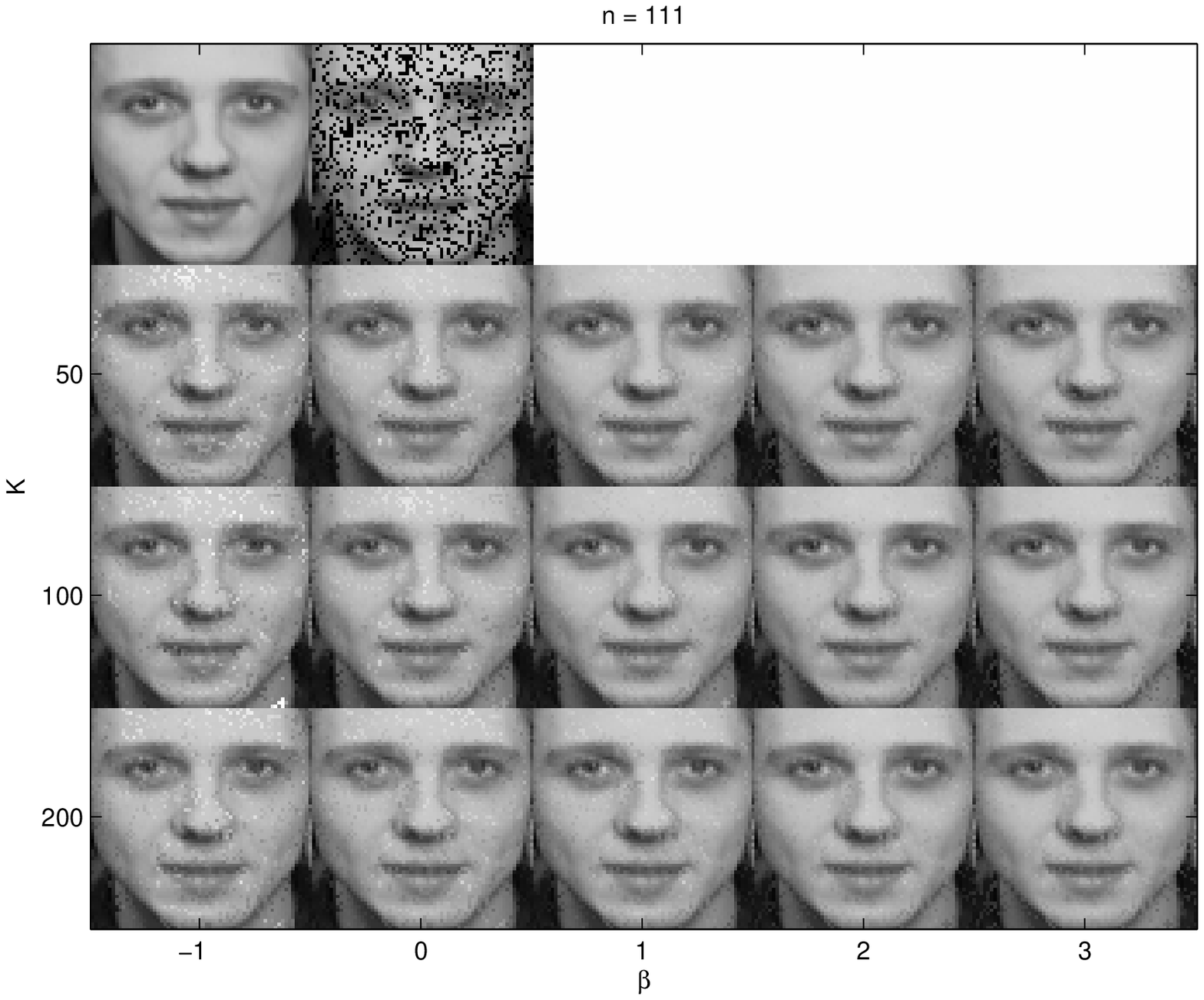} \\
    \medskip
    \hspace{0.3cm}    
   \begin{tabular}{|c|c|c|c|c|c|}
   \hline
   $K$ / $\beta$ & -1 & 0 & 1 & 2 & 3\\
   \hline   
	50 & 28.3  & 30.7 &  32.2  & 31.8 &  31.2 \\
	100 &  26.6 &  30.6 &  32.5  & 33.2 &  32.3 \\
	200 &  29.1 &  31.2  & 32.4 &  32.7  & 32.1 \\
      \hline
   \end{tabular}
   \end{center}
   \end{minipage}
\hspace{0.5cm}
 \begin{minipage}{0.45\linewidth}  
   \begin{center}
    \includegraphics[width=\linewidth]{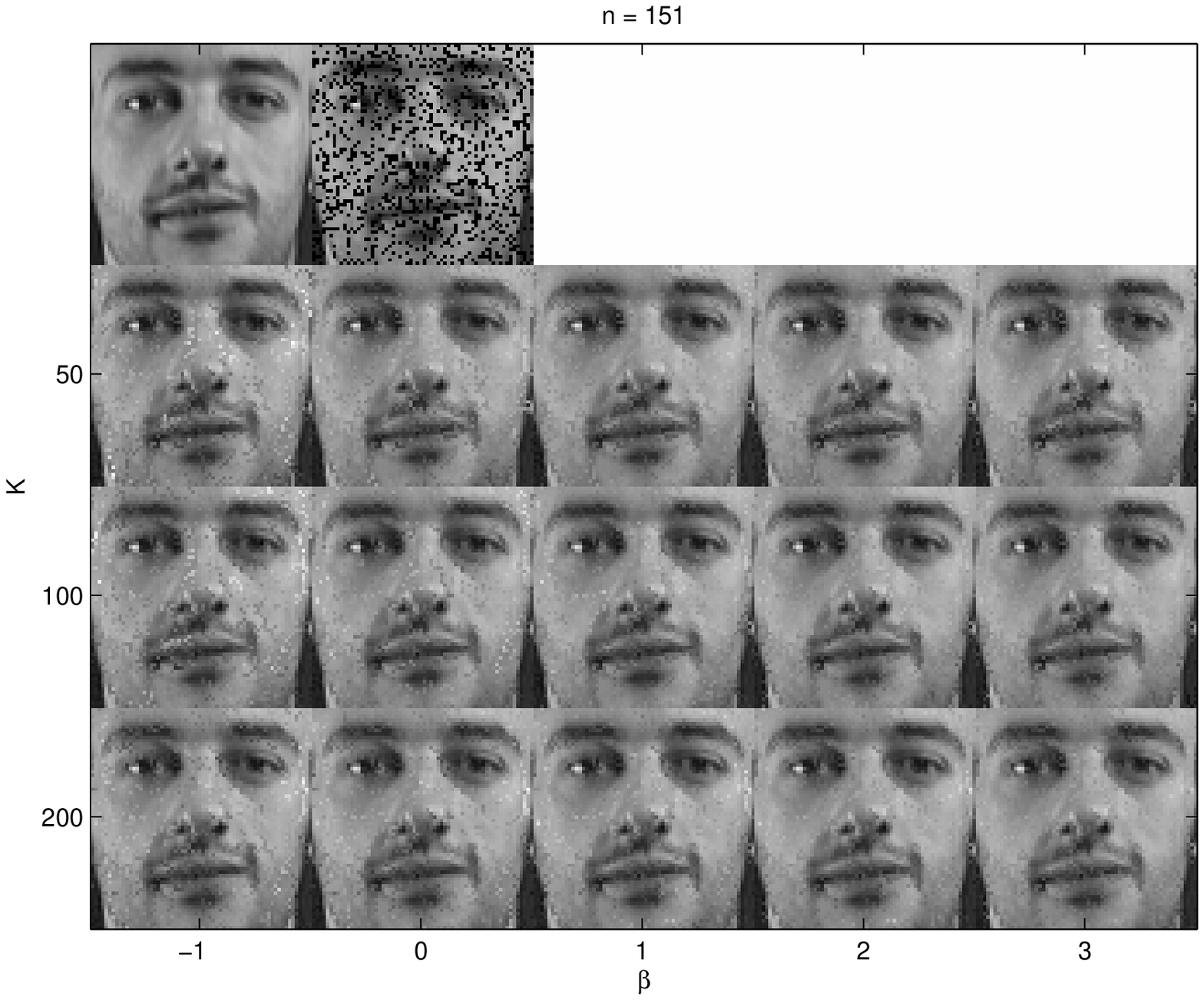} \\
    \medskip
    \hspace{0.3cm}
   \begin{tabular}{|c|c|c|c|c|c|}
   \hline
   $K$ / $\beta$ & -1 & 0 & 1 & 2 & 3\\
   \hline   
    50 & 26.3 & 27.9 & 29.6 & 29.5 & 28.8 \\
    100 & 26.2 & 28.3 & 29.3 & 29.8 & 29.8 \\
    200 & 27.4 & 28.3 & 29.0 & 30.0 & 30.0 \\
      \hline
   \end{tabular}
   \end{center}   
   \end{minipage}
   \end{center}
   \caption{Interpolation results with the Olivetti dataset. Original and corrupted data are shown top left of each plot. Below are the reconstructions obtained for $K \in \{50, 100, 200 \}$ and $\beta \in \{-1,0,1,2,3 \}$. Tables report PSNRs (in dB) of the reconstructions.}
   \label{fig:interp}
\end{figure}

\section{Variants of $\beta$-NMF} \label{sec:variants}

In this section we briefly discuss how some common variants of NMF, penalized NMF and convex-NMF, can be handled under the $\beta$-divergence.

\paragraph{Penalized $\beta$-NMF} 
Supplementary functions of $\W$ and/or $\H$ are often added to the cost function~\eqref{eqn:defcost} so as to induce some sort of regularization of the factor estimates or so as to reflect prior belief, e.g., in Bayesian maximum a posteriori (MAP) estimation. When such penalty terms are separable in the columns of $\H$ or in the rows of $\W$, penalized NMF essentially amounts to solving the following optimization problem:
\begin{equation} \label{eqn:minihpen}
\underset{\hh}{\text{min}} \ C_P(\hh) \defequal D(\vv | \W \hh) +  L(\hh) \ \text{subject to} \ \hh \ge 0
\end{equation}
where $L(\hh)$ is the penalty term. An auxiliary function to $C_P(\hh)$ is readily given by
\bal
G_P(\hh|\tih) \defequal G(\hh|\tih) + L(\hh)
\eal
where $G(\hh|\tih)$ is any auxiliary function to $C(\hh) = D(\vv | \W \hh)$. MM or ME algorithms can then be designed on a case-by-case basis. Let us consider a short example for illustration: $\ell_1$-norm regularization. In that case we have
\bal
L(\hh) = \lambda \sum_k h_k
\eal
where $\lambda$ is a positive weight parameter. Using the auxiliary function designed in Section~\ref{sec:auxfun} and Eq.~\eqref{eqn:gradconvconc}, the gradient of the penalized auxiliary function writes
\balx
\nabla_{h_k} G_L(\hh | \tih) = \sum_f w_{fk} \, \left[ \conv{d}' \left(v_f |  \tilde{v}_f \frac{h_k}{\tilde{h}_k} \right) + \conc{d}'(v_f|\tilde{v}_f) \right] + \lambda. 
\ealx
The MM algorithm for $\ell_1$-regularized $\beta$-NMF takes a very simple form for $\beta \le 1$, such that
\bal \label{eqn:mmsparse}
h_k = \tilde{h}_k \left(  \frac{\sum_f w_{fk} \, v_f \, \tilde{v}_f^{\beta-2} }{\sum_f w_{fk} \, \tilde{v}_f^{\beta-1} + \lambda} \right)^{\gamma(\beta)}.
\eal
This in particular leads to $\ell_1$-regularized NMF algorithms for KL-NMF and IS-NMF with proven monotonicity. An update similar to Eq.~\eqref{eqn:mmsparse} is obtained for $\beta \ge 2$ but the $\lambda$ term appears through its sign opposite at the numerator, instead of appearing at the denominator. Hence the nonnegativity constraint may become active and must be treated carefully; in that case our result coincides with similar findings of \cite{Pauca2006,Morup2007} for the specific case of $\ell_1$-regularized NMF with the Euclidean distance ($\beta = 2$). In the case $\beta \in (1,2)$ the MM algorithm does not come up with a simple closed-form update, which supports the fact in the penalized case handy algorithms may only come on a case-by-case basis. This is similar to Expectation-Maximization (EM) procedures for MAP estimation, in which the E-step is essentially unchanged but where the M-step might become intractable because of the penalty term. {ME algorithms can also be designed for the $\ell_1$-regularized problem and as a matter of fact it can be shown that the results of Table~\ref{tab:betapoly} (i.e., the values of $\beta$ for which a closed-form update exists) still hold in that case. }

\paragraph{Convex $\beta$-NMF} In some recent NMF-related works the dictionary $\W$ is constrained to belong to a known subspace $\ve{S} \in \RR_+^{F \times M}$ such that
\bal
\W = \ve{S} \ve{L}
\eal 
where $\ve{L} \in \RR_+^{M \times K}$. For example \cite{ding10} assume the columns of $\W$ to be linear combinations (with unknown expansion coefficients) of data points (columns of $\V$), so as to enforce the dictionary to be composed of \emph{data centroids}, while \cite{vinc10} assume the dictionary element to be linear combinations of narrow band spectra, so as to enforce harmonicity and smoothness of the dictionary. The term ``convex-NMF'' was introduced by \cite{ding10} to express the idea that $\W$ belongs to the convex set of all nonnegative linear combinations of elements of $\ve{S}$, but this does not make the optimization problem convex in itself, in the general case. \\

{
In this setting, the dictionary update is tantamount to solving
\begin{equation} \label{eqn:miniWconv}
\underset{\ve{L}}{\text{min}} \ C_{cv}(\ve{L}) \defequal D(\V | \ve{S} \ve{L} {\H}) = \sum_{fn} d\left(v_{fn} | \sum_{mk} s_{fm} l_{mk} h_{kn} \right)  \quad \text{subject to} \ \ve{L} \ge 0.
\end{equation}
As a matter of fact, this matricial optimization problem can be turned into vectorial nonnegative linear regression so that the results of Section~\ref{sec:algos} holds. Given some mappings $(f,n) \in \{1,F\} \times \{1,N\} \rightarrow p \in \{1, FN \}$ and $(m,k) \in \{1,M\} \times \{1,K\} \rightarrow q \in \{1, MK \} $ let us introduce the following variables: ${\ve{T}}$ is the matrix of dimension $FN \times MK$ with coefficients $t_{pq} = s_{fm} h_{kn}$, $\underline{\ve{v}}$ is the column vector of size $FN$ with coefficients $\underline{v}_{p} = v_{fn}$, $\underline{\ve{l}}$ is the column vector of size $MK$ with coefficients $\underline{l}_{q} = l_{mk}$. Then we have
\bal
D(\V | \ve{S} \ve{L} {\H}) = \sum_p d \left( \underline{v}_p | \sum_q t_{pq} \, \underline{l}_q \right)
\eal
and thus the estimation of $\ve{L}$ amounts to the approximation $\underline{\ve{v}} \approx \ve{T} \, \underline{\ve{l}}$. As such, any of the algorithms described in Section~\ref{sec:algos} can be employed for this task. As before, the resulting vectorial updates can be turned into matricial updates, leading to simple and efficient implementations. For example, the MM update reads
\bal
\ve{L} \leftarrow \ve{L} . \left( \frac{ \ve{S}^{T} \left[ (\ve{S} \ve{L} \ve{H})^{.(\beta-2)} . \V \right] \H^{T} }{ \ve{S}^{T}  \left[ (\ve{S} \ve{L} \ve{H})^{.(\beta-1)} \right] \H^{T} } \right)^{.\gamma(\beta)}.
\eal
This result proves the monotonicity of some of the algorithms derived heuristically in \citep{vinc10} and also extends the results of \citep{ding10} for convex NMF with the Euclidean distance to the more general $\beta$-divergence.\footnote{More precisely, \cite{ding10} consider a  ``semi''-NMF version where $\ve{S} = \V$ and the data is allowed to be real-valued while the nonnegativity constraint is solely imposed on $\ve{L}$ and $\H$; our results do not apply to this more general framework but only to the special case where $\V$ in nonnegative.}
}

\section{Conclusions} \label{sec:conc}

This paper has addressed NMF with the $\beta$-divergence. The problem may be reduced to a mere nonnegative linear regression problem and our approach is based on the construction of an auxiliary function $G(\hh|\tih)$ which majorizes the objective function $C(\hh)$ everywhere and is tight for $\hh = \tih$. The auxiliary function unifies existing auxiliary functions for the Euclidean distance and the KL divergence \citep{lee01}, for the ``generalized divergence'' of \cite{kom07} (in essence the $\beta$-divergence on its convex part, i.e., $\beta \in [1,2]$) and for the IS divergence \citep{cao99}. Various descent algorithms, free of tuning parameters, may then be derived from this auxiliary function. As such, the findings of this paper may be summarized as follows.

\begin{itemize}
\item The MM algorithm based on the described auxiliary function is shown to yield multiplicative algorithms for $\beta \in \mathbb{R}$, as described by Eq.~\eqref{eqn:upbeta}. For $\beta \in [1,2]$ (interval of values for which the $\beta$-divergence is convex), the MM algorithm coincides with the heuristic algorithm given by Eq.~\eqref{eq:heur_update}, as already known from \cite{kom07}.  
\item In Section~\ref{sec:heur}, we prove the monotonicity of the heuristic algorithm for $\beta \in (0,1)$ by proving the inequality $G(\hh^{\text{H}} | \tih)\leq G(\tih | \tih)$. {Hence, aggregating the existing monotonicity results for $\beta = 0$ and $\beta \in [1,2]$, it can now be claimed that the heuristic algorithm is monotone for $\beta \in [0,2]$, which is the range of values of practical interest that has been considered in the literature.}
\item In Section~\ref{sec:me}, we introduced the concept of maximization-equalization (ME) algorithms. Such algorithms are exhibited for specific values of $\beta$, in particular for $\beta \in \{ 0, 0.5, 1.5, 2 \}$ which are values of practical interest. For $\beta = 0$ (IS divergence) the ME algorithm coincides with the heuristic algorithm, whose monotonicity already holds from \citep{cao99}. For other values of $\beta$ the ME algorithms are nonmultiplicative. For $\beta \in \{ 0.5, 1.5, 2 \}$ they amount to solving polynomial equations of order 1 or 2. The result section has illustrated the faster convergence of the ME approach w.r.t to MM or heuristic, with equivalent complexity.

\item Finally, in Section~\ref{sec:variants} we have considered variants of NMF with the $\beta$-divergence. We have explained how penalty terms may be handled in the auxiliary function setting; in particular we have presented simple multiplicative algorithms for $\ell_{1}$ regularized KL or IS NMF. Then, we have shown how the algorithms constructed for plain NMF holds for convex-NMF, generalizing and proving the monotonicity of existing algorithms.
\end{itemize}

As for perspectives, the present work leaves two important questions unanswered. The first one is the monotonicity of the heuristic algorithm for $\beta \not\in [0,2]$. The monotonicity is observed in practice but we have not been able to come up with proofs in the presented setting. Either other approaches need to be followed or a different type of auxiliary functions than the one presented here needs to be envisaged. As suggested in Section~\ref{sec:defbet}, the convex-concave decomposition of the $\beta$-divergence is not unique and decompositions other than the ``natural'' one employed in this paper may lead to auxiliary functions that more closely fit to the criterion. The second, probably more ambitious question is the convergence of the sequence of iterates produced by the proposed algorithms a stationary point. Partial results exist for Euclidean NMF \citep{lin07a}, convergence of multiplicative rules for nonnegative linear regression (i.e., when only one of the two matrices is updated) has been studied in a few cases, see, e.g., \citep{titt87,depi93,egg98}, but general results for NMF with the $\beta$-divergence are still lacking. A noteworthy attempt has recently been made by \cite{bad10}, which points difficulties in the convergence study due to the inherent scale ambiguity of factorization models.

Finally, another relevant perspective is the design of new types of $\beta$-NMF algorithms. In the Euclidean case, projected gradient methods \citep{lin07b}, second-order active sets methods \citep{kim08}, block-coordinate descent methods \citep{mair10} have recently been shown to outperform standard multiplicative updates, see also \citep{mor09} for a comparison of a selection of algorithms. As such it would be interesting to study how these approaches may extend to the more general $\beta$-NMF framework.

\section*{Acknowledgements} The authors would like to thank Francis Bach, Henri Lantéri, Augustin Lefèvre, Cédric Richard and Céline Theys for inspiring discussions related to optimization in NMF. Many thanks to the reviewers for helpful comments. This work is supported by project ANR-09-JCJC-0073-01 TANGERINE (Theory and applications of nonnegative matrix factorization).


\end{document}